\documentclass[11pt]{article}

\usepackage[english]{babel}
\usepackage[T1]{fontenc}
\usepackage[a4paper,top=3cm,bottom=2cm,left=3cm,right=3cm,marginparwidth=1.75cm]{geometry}

\usepackage{bibspacing}

\usepackage{authblk}
\usepackage{todonotes}
\usepackage{wrapfig}
\usepackage{subcaption}
\usepackage[cal=boondox]{mathalfa}
\usepackage{marginnote}

\usepackage{mathrsfs}
\usepackage{amsmath}
\usepackage{amssymb}
\newcommand{\rom}[1]{%
  \textup{\uppercase\expandafter{\romannumeral#1}}%
}

\usepackage[utf8]{inputenc}
\usepackage{amsthm}
\usepackage{thmtools, thm-restate}

\declaretheorem{theorem}

\newtheorem*{theorem*}{Theorem}

\usepackage{booktabs, adjustbox}
\usepackage{multirow}
\usepackage{threeparttable}
\usepackage{graphicx}
\usepackage{array}
\usepackage{tabularx}

\usepackage[colorlinks=true, allcolors=blue, pdfencoding=auto]{hyperref}

\renewcommand{\H}{\mathcal{K}}

\newcommand{\bx}{{\pmb x}}

\newcommand{\bz}{{\pmb z}}

\newcommand{\balpha}{{\pmb \alpha}}

\let\oldphi\phi
\renewcommand\phi{\operatorname{\oldphi}}

\newcommand\norm[1]{\left\lVert#1\right\rVert}

\usepackage{mathtools}

\usepackage{mathrsfs}
\DeclareMathAlphabet{\mathpzc}{OT1}{pzc}{m}{it}

\newcommand\expect[1]{\mathbb{E}{\left\lbrack#1\right\rbrack}}
\newcommand\expectp[2]{\mathbb{E}_{#1}{\left\lbrack#2\right\rbrack}}

\theoremstyle{definition}

\usepackage[toc,page]{appendix}
\usepackage{mathtools}
\usepackage{listings}
\usepackage{tikz}
\usepackage[nomessages]{fp}
\usepackage{adjustbox}
\usepackage{outlines}
\usepackage{times}
\usepackage{graphicx,color}
\usepackage{array,float}
\usepackage{url}
\usepackage{amstext,amssymb,amsmath}
\usepackage{hyphenat}
\usepackage{amsthm}
\usepackage{verbatim}
\usepackage{bm}
\usepackage{paralist}
\usepackage{ulem}

\newcommand{\sign}{\mathop{\mathrm{sign}}}

\usepackage{etoolbox}

\newcommand{\surround}[2][r]%
  {\ifstrequal{#1}{round}%
    {\left( #2 \right)}%
    {\ifstrequal{#1}{square}%
      {\left[ #2 \right]}%
      {\ifstrequal{#1}{curly}%
        {\left\{ #2 \right\}}%
        {\ifstrequal{#1}{angle}%
          {\left\langle #2 \right\rangle}%
          {\ifstrequal{#1}{|}%
            {\left\lvert #2 \right\rvert}%
            {\ifstrequal{#1}{||}%
              {\left\lVert #2 \right\rVert}%
              {\ifstrequal{#1}{floor}%
                {\left\lfloor #2 \right\rfloor}%
                {\ifstrequal{#1}{ceil}%
                  {\left\lceil #2 \right\rceil}%
                  {\ifstrequal{#1}{.}%
                    {\left. #2 \right.}%
                    {\left( #2 \right)}%
                  }%
                }%
              }%
            }%
          }%
        }%
      }%
    }%
  }

\renewcommand{\H}{\mathcal{H}}
\newcommand{\R}{\mathbb{R}}
\declaretheorem{Proposition}
\title{To Understand Deep Learning We Need to Understand Kernel Learning}
 \author[]{Mikhail Belkin}
 \author[]{~~Siyuan Ma}
 \author[]{~~Soumik Mandal}
 \affil{Department of Computer Science and Engineering}
 \affil{Ohio State University}
 \affil{
\textit{\{mbelkin, masi\}@cse.ohio-state.edu},
 \textit{mandal.32@osu.edu}
 }

\author{}
\date{}

\begin{document}

\maketitle

\vspace{-.7in}
\begin{abstract}
Generalization performance of classifiers in deep learning has recently become a subject of intense study. Deep models, which are typically heavily over-parametrized, tend to fit the training data exactly.  Despite this ``overfitting", they perform well on test data, a  phenomenon  not yet fully understood.

The first point of our paper is that strong performance of overfitted classifiers is not a unique feature of deep learning. Using six real-world  and two synthetic datasets, we establish experimentally that kernel machines trained to have zero classification error  or  near zero regression error (interpolation) perform very well on test data, even when the labels are corrupted with a high level of noise. 
We proceed to give a lower bound on the norm of zero loss solutions for smooth kernels, showing that they increase nearly exponentially with  data size. We point out that this is difficult to reconcile with the existing generalization bounds.  Moreover,  none of the  bounds produce non-trivial results for interpolating solutions.

Second, we  show experimentally that (non-smooth) Laplacian kernels easily fit random labels, 
a finding that parallels results recently reported for ReLU neural networks. In contrast, fitting noisy data requires many more epochs for smooth Gaussian kernels.  Similar  performance of overfitted Laplacian and Gaussian classifiers on test, suggests that  generalization is tied to the properties of the kernel function  rather than the  optimization process.

Certain  key phenomena of  deep learning are manifested similarly in  kernel methods in the modern ``overfitted" regime. 
The combination of the experimental and theoretical results presented in this paper indicates a need for  new theoretical ideas for understanding properties of classical kernel methods. 
We argue that progress on understanding  deep learning will be difficult until more  tractable ``shallow'' kernel methods are better understood.
\end{abstract}

\section{Introduction}


The key question in supervised machine learning is that of {\it generalization}. How will a classifier trained on a certain data set perform on unseen data?   A typical theoretical setting for addressing this question is classical Empirical Risk Minimization (ERM)~\cite{Vapnik}. 
Given data $\{(\bx_i,y_i), i=1,\ldots,n\}$ sampled from a probability distribution $P$ on $\Omega \times \{-1,1\}$, a class of functions ${\mathcal H}:\Omega \to \R$ and a loss function $l$, ERM finds a minimizer  of the empirical loss:
$$
f^* = \arg\min_{f \in {\mathcal H~}} L_{emp}(f) := \arg\min_{f \in {\mathcal H~}} \sum_i l(f(\bx_i),y_i) 
$$
Most approaches work by controlling and analyzing the capacity/complexity of the space $\H$. 
Many mathematical measures of function space complexity exist, including VC and fat shattering dimensions, Rademacher complexity, covering numbers (see, e.g.,~\cite{anthony2009neural}). These analyses generally  yield bounds on {\it the generalization gap}, i.e., the difference between the empirical and expected loss of classifiers. Typically, it is shown  
that the generalization gap tends to zero at a certain rate as the number of points $n$ becomes large. For example,  many of the classical bounds on the generalization gap are of the form $\left | \expect{l(f^*(\bx),y)} - L_{emp}(f^*) \right | < O^*(\sqrt{c/n})$, where $c$ is a measure of complexity of $\H$, such as VC-dimension.  Other methods, closely related to ERM, include regularization to control bias/variance (complexity) trade-off by  parameter choice, and  result in similar bounds.  Closely related implicit regularization methods, such as early stopping for gradient descent~\cite{yao2007early,raskutti2014early,camoriano2016nytro}, provide regularization by limiting the amount of computation, thus aiming to achieve better performance at a lower computational cost. 
All of these  approaches suggest trading off  accuracy (in terms of some loss function) on the training data to get performance guarantees on  the unseen test data.

In recent years we have seen impressive progress in supervised learning due, in particular, to deep neural architectures. 
These networks employ large numbers of parameters, often exceeding the size of training data by several orders of magnitude~\cite{canziani2016analysis}. This over-parametrization allows for convergence to global optima, where the training  error is zero or nearly zero. Yet these ``overfitted\footnote{We use {\it overfitting} as a purely technical term to refer to zero classification error as opposed to interpolation which has zero regression error.}'' or even interpolated networks still generalize well to  test data, a situation which seems difficult to reconcile with available theoretical analyses (as observed, e.g., in~\cite{zhang2016understanding} or, much earlier, in~\cite{breiman2018reflections}). 
There have been a number of recent efforts to understand generalization and overfitting in deep networks including~\cite{bartlett2017spectrally,liang2017fisher,poggio2018theory}. 

In this paper we  make the case that progress on understanding deep learning is unlikely to move forward until similar phenomena in classical kernel machines are recognized and understood. Kernel machines can be viewed as linear regression in infinite dimensional Reproducing Kernel Hilbert spaces (RKHS), which correspond to positive-definite kernel functions, such as Gaussian or Laplacian kernels. They can also be interpreted as two-layer neural networks with a fixed first layer.  As such, they are far more amenable to theoretical analysis than arbitrary deep networks. Yet,  despite numerous  observations  in the literature that very small values of regularization parameters (or even direct minimum norm solutions) often result in optimal performance~\cite{shalev2011pegasos, takac2013mini, zhang2016understanding,gonen2016solving,rudi2017falcon}, the systematic nature of near-optimality of kernel classifiers trained to have
zero classification error  or zero regression error  has not been recognized. We note that margin-based analyses, such as those proposed to analyze overfitting in boosting~\cite{schapire1998}, do not easily explain performance of interpolated classifiers in the presence of  label noise, as sample complexity must scale linearly with the number of data points.
We would like to point out an insightful (but seemingly little noticed) recent paper~\cite{wyner2017explaining} which proposed an alternative explanation for the success of Adaboost, much closer to our discussion here. 

Below we will show that most  bounds for smooth kernels will, indeed, diverge with increasing data. On the other hand,
empirical evidence shows consistent and robust generalization performance of ``overfitted" and interpolated classifiers even for high label noise levels.

We will discuss these and other related issues in detail, providing both theoretical results and empirical data.   
The contribution of this paper are as follows: 
\begin{itemize} 
\item {\bf Empirical properties of overfitted and interpolated  kernel classifiers.}\\
1. The phenomenon of strong generalization performance of ``overfitted"/interpolated classifiers is not unique to deep networks. We demonstrate experimentally that kernel classifiers that  have zero classification or regression error on the training data, still perform well on test. We use six real-world datasets  (Section~\ref{sec:interpolated}) as well as two synthetic datasets (Section~\ref{sec:norm_bound}) to demonstrate the ubiquity of this behavior. Additionally, we observe  that regularization by early stopping provides at most a minor improvement to classifier performance.\\
\noindent 2. It was recently observed in~\cite{zhang2016understanding} that 
ReLU networks trained with SGD  easily fit standard datasets with random labels, requiring only about three times as many epochs as for fitting the original labels. Thus  the fitting capacity of ReLU network function space reachable by a small number of  SGD steps is very high. In Section~\ref{sec:fit-noise} we demonstrate very similar behavior
exhibited by (non-smooth) Laplacian (exponential)  kernels, which are easily able to fit random labels.
In contrast,  as expected from the theoretical considerations of fat shattering dimension~\cite{belkin2018approximation}, it is far more computationally difficult to fit random labels using Gaussian kernels. However, we observe that the actual test performance of interpolated Gaussian and Laplacian kernel classifiers on real and synthetic data is very similar, and remains similar even with added label noise. 


\item  {\bf Theoretical results and the supporting experimental evidence.} In Section~\ref{sec:norm_bound} we show theoretically that  performance of interpolated kernel classifiers cannot be explained by the existing generalization bounds available for kernel learning.  Specifically, we prove lower bounds on the RKHS norms of overfitted solutions for smooth kernels, showing that they must increase nearly exponentially with the data size. Since most available generalization bounds depend  polynomially on the norm of the solution, this result implies divergence of most bounds as data goes to infinity. Moreover, to the best of our knowledge, none of the existing bounds (including potential logarithmic bounds) apply to interpolated (zero regression loss) classifiers.

Note that we need an assumption that  the loss of the Bayes optimal classifier (the label noise) is non-zero. While it is usually believed that most real data have some level of label noise, it is not usually possible to ascertain this  is the case. We address this issue in two ways by analyzing (1) synthetic datasets with a known level of label noise (2) real-world datasets with additional random label noise. In both cases we see that empirical test performance of interpolated kernel classifiers decays at slightly below  the noise level, as it would, if the classifiers were nearly optimal. This finding holds even for very high levels of label noise.
We thus conclude that the existing bounds are unlikely to  provide insight into the generalization performance of  kernel classifiers. Moreover, since the empirical risk is zero, any potential non-trivial bound for the generalization gap, aiming to describe  noisy data, must have tight constants to produce a value between the (non-zero) Bayes risk and $1$. To the best of our knowledge, no examples of such bounds exist.

\end{itemize}

We will now discuss some important points, conclusions and conjectures  based on the combination of theoretical and experimental results presented in this paper.\\
\noindent{\bf Parallels between deep and shallow architectures in performance of overfitted classifiers.}
There is extensive empirical evidence, including the experiments in our paper,
that ``overfitted" kernel classifiers 
demonstrate strong performance on a range of datasets.
Moreover, in Section~\ref{sec:interpolated} we see that introducing regularization (by early stopping) provides at most a modest improvement to the classification accuracy. 
Our findings  parallel those for deep networks discussed in~\cite{zhang2016understanding}. Considering that kernel methods can be viewed as a special case of two-layer neural network architectures, 
we conclude that deep network structure, as such, is unlikely to play a significant role in this surprising phenomenon.\\
\noindent{\bf Existing bounds for kernels lack explanatory power in overfitted regimes.}
Our experimental results show that kernel classifiers demonstrate nearly optimal performance even when the label noise is known to be significant. On the other hand,  the existing bounds for overfitted/interpolated kernel methods diverge  with increasing data size in the presence of label noise. We believe that a new theory of kernel methods, not dependent on norm-based concentration bounds, is needed to understand this behavior.  

At this point we know of few candidates for such a theory. A notable (and, to the best of our knowledge, the only) example is $1$-nearest neighbor classifier, with expected loss that can be bounded asymptotically by twice the Bayes risk~\cite{cover1967nearest}, while its empirical loss (both classification and regression) is identically zero. 
We conjecture that similar ideas are needed to analyze  kernel methods and, potentially, deep learning.\\
\noindent{\bf Generalization and optimization.}
We observe that smooth Gaussian kernels and non-smooth Laplacian kernels have very different optimization properties. 
We show experimentally that (less smooth) Laplacian kernels easily fit standard datasets with random labels,  requiring only about twice the number of epochs needed to fit the original labels (a finding that closely parallels results recently reported for ReLU neural networks in~\cite{zhang2016understanding}). In contrast  (as suggested by the theoretical considerations of fat shattering dimension in~\cite{belkin2018approximation}) optimization by gradient descent is far more computationally demanding for  (smooth) Gaussian kernels.  On the other hand, test performance of kernel classifiers is very similar for Laplacian and Gaussian kernels, even with added label noise. Thus the generalization performance  of classifiers appear to be related to the  structural properties of the kernels (e.g., their radial structure) rather than their properties with respect  to the optimization methods, such as SGD. \\
\noindent{\bf Implicit regularization and loss functions.}
One proposed explanation for the performance of deep networks is the idea of implicit regularization introduced by methods such as early stopping in gradient descent~\cite{yao2007early,raskutti2014early,neyshabur2014search,camoriano2016nytro}. These approaches suggest trading off some accuracy on the training data  by limiting the amount of computation, to get better performance on  the unseen test data.
It can be shown~\cite{yao2007early} 
that for kernel methods early stopping for gradient descent is effectively equivalent to traditional regularization methods, such as Tikhonov regularization. 

As interpolated kernel methods fit the labels exactly (at or close to numerical precision), implicit regularization, viewed as a  trade-off between train and test performance, cannot provide an explanation for their generalization performance. 
While overfitted (zero classification loss) classifiers can, in principle, be taking advantage of regularization by introducing regression loss not reflected in the classification error (cf.~\cite{schapire1998}), we see (Section~\ref{sec:interpolated},\ref{sec:norm_bound}) that their performance does not significantly differ from that for interpolated classifiers for which margin-based explanations to not apply. 

Another interesting point is that any strictly convex loss function leads to the same interpolated solution. Thus, it is unlikely that the choice of loss function relates to the generalization properties of classifiers\footnote{It has been long noticed that performance of kernel classifiers does not significantly depend on the choice of loss functions. For example, kernel SVM performs very similarly to kernel least square regression~\cite{zhang2004SVM}.}.

Since deep networks are also trained to fit the data exactly,  the similarity to kernel methods suggests that implicit regularization or the specifics of the loss function used in training, are not the basis of their  generalization properties. \\
\noindent{\bf Inductive bias and minimum norm solutions.} While the notions of {\it regularization} and {\it inductive bias} are frequently used interchangeably in the literature, we feel it would be useful to draw a distinction between regularization which introduces a bias on the training data and  inductive bias, which gives preferences to certain functions without affecting their output on the training data. 

While interpolated  methods fit the data exactly and thus produce no regularization,  minimum RKHS norm interpolating solutions introduce inductive bias by choosing functions with special properties. Note that infinitely many RKHS functions are capable of interpolating the data\footnote{Indeed, the space of RKHS interpolating functions is dense in the space of all functions in $L^2$!}. However,  the Representer Theorem~\cite{aronszajn1950theory}  ensures that the minimum norm interpolant is a linear combination of kernel functions supported on data points $\{K(\bx_1,\cdot),\ldots,K(\bx_n,\cdot)\}$. As we observe from the empirical results, these solutions have special generalization properties, which cannot be expected from arbitrary interpolants.  While we do not yet understand how this inductive bias leads to strong generalization properties of kernel interpolants, they are obviously related to the structural properties of kernel functions and their RKHS. It is instructive to compare this setting to 1-NN classifier. While no guarantee can be given for piece-wise constant interpolating functions in general, the specific piece-wise constant function chosen by 1-NN has certain optimality properties, guaranteeing the generalization error of at most twice the Bayes risk~\cite{cover1967nearest}. 

It is well-known that gradient descent (and, in fact, SGD) for any strictly convex loss, initialized at $0$ (or any point other point within the span of $\{K(\bx_1,\cdot),\ldots,K(\bx_n,\cdot)\}$),  converges to the minimum norm solution, which is the unique interpolant for the data within the span of the kernels functions.  On the other hand, it can be easily verified\footnote{The component of the initialization vector orthogonal to the span does not change with the iterative updates.} that GD/SGD initialized outside of the span of $\{K(\bx_1,\cdot),\ldots,K(\bx_n,\cdot)\}$ cannot converge to the minimum RKHS norm solution. 
Thus  the inductive bias corresponding to SGD with initialization at zero, is consistent with that  of the minimum norm solution. 

This view also provides a natural link to the phenomenon observed in AdaBoost training, where the test error  improves even after the classification error on train reached zero~\cite{schapire1998}. 
If we believe that the minimum norm solution (or the related maximum margin solution) has special properties, iterative optimization should progressively improve the classifier, regardless of the training set performance. Furthermore, based on this reasoning, generalizations bounds that connect empirical and expected error are unlikely to be helpful.

Unfortunately, we do not have an analogue of the Representer Theorem for deep networks.  Also, despite a number of recent attempts (see, e.g.,~\cite{neyshabur2017exploring}), it is not clear how best to construct a norm for deep networks similar to the RKHS norm for kernels. Still,  it appears likely that similarly to kernels,  the structure of neural networks in combination with algorithms, such as SGD, introduce an inductive bias\footnote{We conjecture that fully connected neural networks have inductive biases similar to those of kernel methods. 
On the other hand, convolutional networks seem to have strong inductive biases tuned to vision problems, which can be used even in the absence of labeled data~\cite{ulyanov2017deep}. }. 

We see that kernel machines have a unique analytical advantage over other powerful non-linear techniques such as boosting and deep neural networks as their minimum norm solutions can be computed analytically and analyzed using a broad range of mathematical analytic techniques. Additionally, at least for smaller data, these solutions can be computed using the classical direct methods for solving systems of linear equations. We argue that kernel machines provide a natural analytical and experimental platform for understanding inference in modern machine learning. 

\noindent{\bf A remark on the importance of accelerated algorithms, hardware and SGD.}
Finally, we note that the experiments shown in this paper, particularly fitting noisy labels with Gaussian kernels,  would be difficult to conduct without fast kernel training algorithms (we used EigenPro-SGD~\cite{ma2017diving}, which provided 10-40x acceleration over the standard SGD/Pegasos~\cite{shalev2011pegasos}) combined with  modern GPU hardware. 
By a remarkably serendipitous coincidence, small mini-batch SGD can be shown to be  exceptionally effective (nearly $O(n)$ more effective than full gradient descent) for interpolated classifiers~\cite{ma2017interpolation}. 


\smallskip

To summarize, in this paper we demonstrate significant parallels between the properties of deep neural networks and the classical kernel methods trained in the ``modern'' overfitted regime. Note that kernel methods can be viewed as a  special type of two-layer neural networks with a fixed first layer. Thus, we argue that more complex deep networks are unlikely to be amenable to analysis unless simpler and analytically more tractable kernel methods are better  understood. Since the existing bounds seem to provide little explanatory power for their generalization performance, new insights and mathematical analyses are needed.


\section{Setup}

We recall some  properties of kernel methods used in this paper. Let $K(\bx,\bz):\R^d\times\R^d\to \R$ be a positive definite kernel. Then there exists a corresponding Reproducing Kernel Hilbert Space $\H$ of functions on $\R^d$, associated to the kernel $K(x,z)$.
Given a data set $\{(\bx_i, y_i), i=1,\ldots,n\}, \bx_i\in \R^d, y_i \in \R$, let $K$ be the associated kernel matrix, $
K_{ij}=K(\bx_i,\bx_j)$ and define the minimum norm interpolant
\begin{equation}\label{eq:min_norm}
f^* ={\arg\min}_{f\in \H, ~f(\bx_i)=y_i} \|f\|_\H 
\end{equation}
Here  $\|f\|_\H$ is the RKHS norm of $f$. From the classical representer theorem~\cite{aronszajn1950theory} it follows that  $f^*$ exists (as long as no two data points $x_i$ and $x_j$ have the same features but different labels). 
Moreover, $f^*$ can be written explicitly as 
\begin{equation}\label{eq:interpolant}
f^*(\cdot) = \sum \alpha^*_i K(\bx_i,\cdot),\,\mathrm{where} ~~ (\alpha^*_1,\ldots,\alpha^*_n)^T = K^{-1} (y_1,\ldots, y_n)^T
\end{equation}
The fact that matrix $K$ is invertible follows directly from the positive definite property of the kernel. It is easy to verify  that indeed $f(\bx_i) = y_i$ and hence the function $f^*$ defined by Eq.~\ref{eq:interpolant} {\it interpolates} the data. 

An equivalent way of writing Eq.~\ref{eq:min_norm} is to observe that $f^*$ minimizes $\sum l(f(\bx_i), y_i)$ for any non-negative loss function $l(\tilde{y},y)$, such that $l(y,y)=0$. If $l$ is strictly convex, e.g., the square loss $l(f(\bx_i),y_i) = (f(\bx_i)-y_i)^2$, then $\balpha^*$ is the unique vector satisfying
\begin{equation}
\balpha^*=\arg \min_{\alpha\in \R^n} \sum_{i=1}^n l\left(\left(\sum_{j=1}^n \alpha_i K(\bx_j,\bx_i)\right),y_i\right) 
\end{equation}

This is an important formulation as it allows us to define $f^*$ in terms of an unconstrained optimization problem on a finite-dimensional space $\R^n$.  In particular, iterative methods 
can be used to solve for $\alpha^*$, often obviating the need to  invert the $n \times n$ matrix $K$. Matrix inversion generally requires $n^3$ operation, which is prohibitive for large data. 

We also recall that the RKHS norm of an arbitrary function of the form $f(\cdot)=  \sum \alpha_i K(\bx_i,\cdot)$ can be easily computed as
$$
\|f\|_H^2 = \langle \balpha , K \balpha\rangle = \sum_{ij} \alpha_i K_{ij} \alpha_j
$$

In this paper we will primarily use the  popular smooth  Gaussian kernel $K(\bx,\bz)=\exp{\left(-\frac{\|\bx-\bz\|^2}{2\sigma^2}\right)}$ as well as non-smooth Laplacian  (exponential) kernel $K(\bx,\bz)=\exp{\left(-\frac{\|\bx-\bz\|}{\sigma}\right)}$. 
We will use both direct linear systems solvers and iterative methods.\\
\paragraph{Interpolation versus ``overfitting''.} In this paper we will refer to  
 classifiers as {\it interpolated} if their square loss on the training error is zero or close to zero.  We will call classifiers {\it overfitted} if the same holds for classification loss (for the theoretical bounds we will additionally require a small fixed margin on the training data). Notice that while interpolation implies overfitting, the converse does not hold.


\section{Generalization Performance of Overfitted/Interpolating Classifiers}
\label{sec:interpolated}

\begin{figure}[!ht]
  \centering
  \begin{minipage}[l]{0.32\textwidth}
    \includegraphics[width=\textwidth]{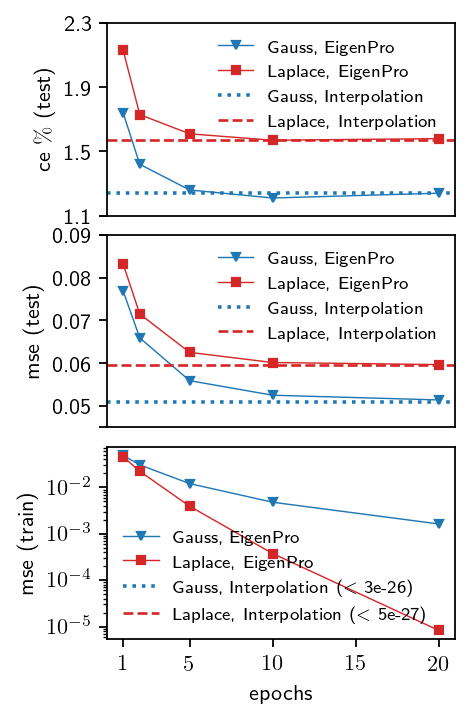}
    \subcaption{MNIST}
  \end{minipage}
  \hfill
  \begin{minipage}[l]{0.32\textwidth}
    \includegraphics[width=\textwidth]{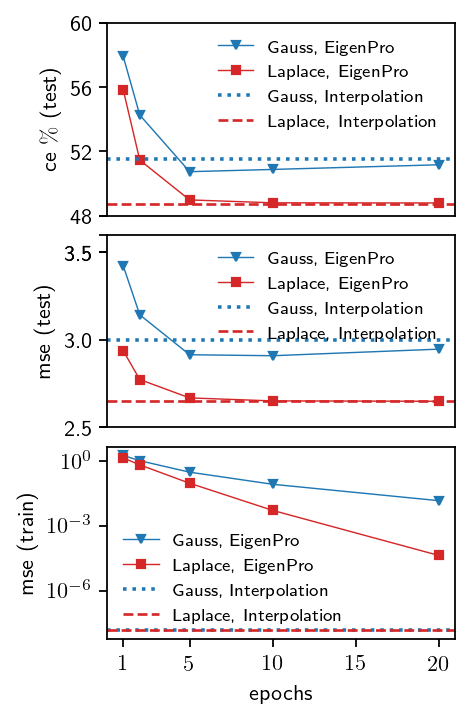}
    \subcaption{CIFAR-10}
  \end{minipage}
  \hfill
  \begin{minipage}[l]{0.32\textwidth}
    \includegraphics[width=\textwidth]{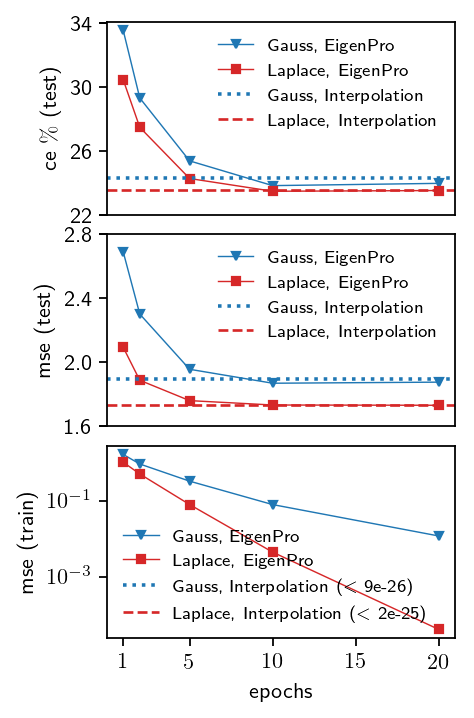}
    \subcaption{SVHN ($2 \cdot 10^4$ subsamples)}
  \end{minipage}
  
  \begin{minipage}[l]{0.32\textwidth}
    \includegraphics[width=\textwidth]{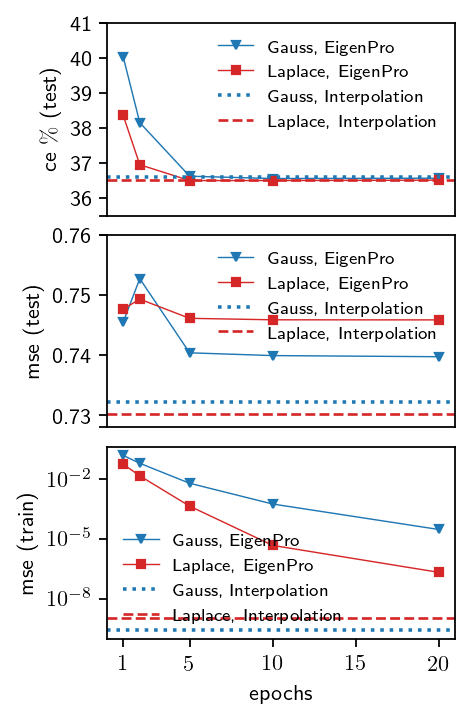}
    \subcaption{TIMIT ($5 \cdot 10^4$ subsamples)}
  \end{minipage}
  \hfill
  \begin{minipage}[l]{0.32\textwidth}
    \includegraphics[width=\textwidth]{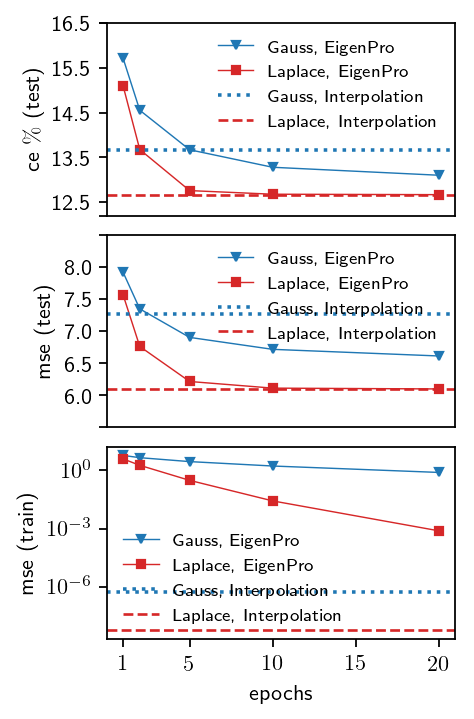}
    \subcaption{HINT-S ($2 \cdot 10^4$ subsamples)}
  \end{minipage}
  \hfill
  \begin{minipage}[l]{0.31\textwidth}
     \includegraphics[width=\textwidth]{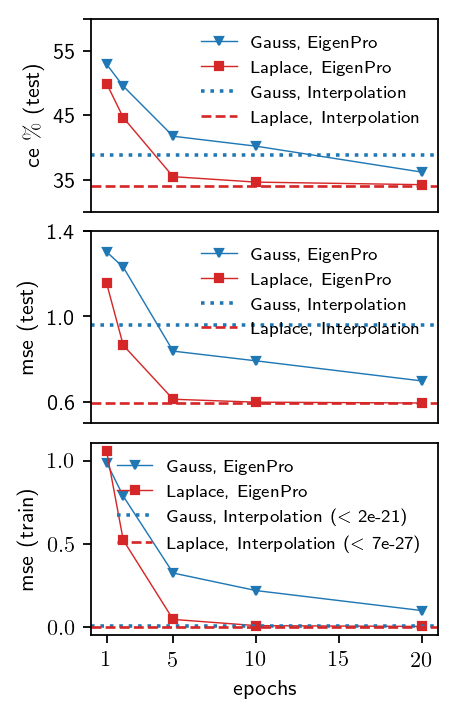}
    \subcaption{20 Newsgroups}
  \end{minipage}
  \caption{Comparison of approximate classifiers trained by EigenPro-SGD~\cite{ma2017diving} and interpolated classifiers obtained from direct method for kernel least squares regression.}
  \label{fig:interpolate}
  \vspace{-4mm}
  \begin{flushleft}
  \small $\dagger$ All methods achieve $0.0\%$ classification error on training set. $\ddagger$ We use subsampled dataset to reduce the computational complexity and to avoid numerically unstable direct solution.  
  \end{flushleft}
\end{figure}

In this section we establish empirically that interpolating kernel methods provide strong performance on a range of standard datasets (see Appendix~\ref{app:expr-setup} for dataset descriptions) both in terms of regression and classification. To construct kernel classifiers we use iterative EigenPro-SGD method~\cite{ma2017diving}, which is an accelerated version of SGD in the kernel space (cf.~Pegasos~\cite{shalev2011pegasos}). This provides a highly efficient implementation of kernel methods and, additionally, a setting parallel to neural net training using SGD. Our experimental results are summarized in Fig.~\ref{fig:interpolate} (see Appendix~\ref{app:exper-results} for 
full numerical results including the classification accuracy on the training set).

We see that as the number of epochs increases, training square loss ({\bf mse}) approaches zero\footnote{The training classification error (not shown), is similarly small. After  20 epochs of EigenPro it is zero for all datasets, except for 20 Newsgoups with Gaussian/Laplace kernels and HINT-S with Gaussian kernel (see Appendix~\ref{app:exper-results}).}. On the other hand, the test error, both regression ({\bf mse}) and classification ({\bf ce}) remains very stable and, in most cases (in all cases for Laplacian kernels), keeps decreasing and then stabilizes. We thus observe that early stopping regularization~\cite{yao2007early,raskutti2014early}  provides a  small or no benefit in terms of either classification or regression error. 

For comparison, we also show the performance of  interpolating solutions given by  
Eq.~\ref{eq:interpolant} and solved using {\it direct methods}.
As expected, direct solutions always provide a highly accurate interpolation for the  training data with the error in most cases close to numerical precision. Remarkably, we see that in all cases performance of the interpolated solution on {\it test} is either optimal or close to optimal both in terms of both regression  and classification error. 

Performance of overfitted/interpolated  kernel classifiers closely parallels behaviors of deep  networks noted in~\cite{zhang2016understanding} which fit the data exactly (only the classification error is reported there, other references also report MSE~\cite{chaudhari2016entropy, huang2016densely, sagun2017empirical, bartlett2017spectrally}). 
We note that  observations of unexpectedly strong performance of overfitted classifiers have been made before. For example, in kernel methods it has been observed on multiple occasions that very small values of regularization parameters frequently lead to optimal performance~\cite{shalev2011pegasos, takac2013mini}. 
Similar observations  were also made for Adaboost and Random Forests~\cite{schapire1998} (see~\cite{wyner2017explaining} for a recent and quite different take on that).
However, we have not seen recognition or systematic exploration of this (apparently 
ubiquitous) phenomenon for kernel methods, and, more generally, in connection to interpolated classifiers and generalization with respect to the square loss.

In the next section we  examine in detail why the existing margin bounds are not likely to provide  insight into the generalization properties of classifiers in  overfitted and interpolated regimes. 

\section{Existing  Bounds   Provide No Guarantees for Interpolated Kernel Classifiers}
\label{sec:norm_bound}

In this section we  discuss  theoretical considerations related to  generalization bounds for kernel classification and regression corresponding to smooth kernels. We also provide  further supporting experimental evidence. 
Our main theoretical result shows that the norm of overfitted kernels classifiers increases nearly exponentially with the data size as long as the error of the Bayes optimal classifier (the label noise) is non-zero.  Most of the available generalizations bounds depend at most polynomially on the RKHS norm, and hence diverge to infinity as data size increases and none apply to interpolated classifiers. On the other hand, we will see that the empirical performance of interpolated classifiers remains nearly optimal, even with added label noise.  

Let $(\bx_i,y_i) \in \Omega\times \{-1,1\}$ be a labeled dataset, $\Omega\subset \R^d$ a bounded domain, and let the data be chosen from some probability measure $P$ on $\Omega\times \{-1,1\}$. We will assume that the loss of the Bayes optimal classifier (the label noise) is not $0$, i.e.,  $y$ is not a deterministic function of $\bx$ on a subset of non-zero measure.


We will say that $h \in \H$ {\it $t$-overfits} the data, if it achieves zero classification loss, and, additionally,
$\forall_i y_i h(\bx_i) > t >0$ for at least a fixed portion of the training data.  This condition is necessary as zero classification loss classifiers with arbitrarily small norm can be obtained by simply scaling any interpolating solution. 
The margin condition is far weaker than interpolation, which requires $h(\bx_i) = y_i$ for all data points. 

We now provide a lower bound on the function norm of $t$-overfitted classifiers in RKHS corresponding to Gaussian kernels\footnote{The results also apply to other classes of smooth kernels, such as inverse multi-quadrics.}.

\begin{theorem}{\label{th:norm}} Let $(\bx_i,y_i),i=1,\ldots,n$ be data sampled from $P$ on $\Omega\times \{-1,1\}$. Assume that $y$ is not a deterministic function of $x$ on a subset of non-zero measure.  Then, with high probability, any $h$  that $t$-overfits the data, satisfies
$$
\|h\|_\H > Ae^{B\,n^{1/d}}
$$
for some constants  $A,B>0$ depending on $t$. 
\end{theorem}
\begin{proof} 
Let $B_R = \{f \in \H, \|f\|_\H < R\} \subset \H$ be a ball of radius $R$ in the RKHS $\H$. We will prove that with high probability $B_R$ contains no functions that  $t$-overfit the data, unless $R$ is large, which will imply our result. 

Let $l$ be the hinge loss with margin $t$: $l(f(\bx),y)=(t - yf(\bx))_{+}$. 
Let $V_\gamma(B_R)$ be the fat shattering dimension of the function space $B_R$ with the parameter $\gamma$. By the classical results on  fat shattering dimension (see,e.g.,\cite{anthony2009neural})
$\exists{C_1, C_2 > 0}$ such that with high probability $\forall_{f\in B_R}$:
$$
\left|\frac{1}{n}\sum_i l(f(\bx_i),y_i) - \expectp{P}{l(f(\bx),y)}\right| \le C_1\gamma + C_2\sqrt{\frac{V_\gamma (B_R)}{n}}
$$ 
Since $y$ is not a deterministic function of $x$ on some subset of non-zero measure, $\expectp{P}{l(f(\bx),y)}$ is non-zero. Fix $\gamma$ to be a positive number, such that  $C_1 \gamma < \expectp{P}{l(f(\bx),y)}$.

Suppose now that a function $h \in B_R$ $t$-overfits the data. Then  $\frac{1}{n}\sum_i l(h(\bx_i),y_i) = 0$ and hence 
$$
0< \expectp{P}{l(f(\bx),y)} - C_1\gamma < C_2\sqrt{\frac{V_\gamma (B_R)}{n}}
$$ 

Thus the ball $B_R$ with high probability contains no function that $t$-overfits the data unless
$$
{V_\gamma (B_R)} > \frac{n}{C_2}\left(\expectp{P}{l(f(\bx),y)} - C_1\gamma\right)^2
$$

On the other hand,~\cite{belkin2018approximation} gives a bound on the $V_\gamma$ dimension of the form $V_\gamma (B_R) < O\left(\log^d\left(\frac{R}{\gamma}\right)\right)$. 
Expressing $R$ in terms of $V_\gamma (B_R)$, we see that $B_R$ with high probability contains no function that $t$-overfits the data unless 
$R$ is at least  $Ae^{B\,n^{1/d}}$ for some $A,B>0$. That completes the proof.
\end{proof}

\noindent{\bf Remark.} The bound in~Eq.~\ref{eq:min_norm} applies to any $t$-overfitted classifier, independently of the algorithm or loss function. 

We will now briefly discuss the bounds available for kernel methods. Most of the available bounds for kernel methods (see, e.g., \cite{steinwart2008support, rudi2015less}) are of the following (general) form:
\begin{equation*}
\begin{split}
\left|\frac{1}{n}\sum_i l(f(\bx_i),y_i) - \expectp{P}{l(f(\bx),y)} \right| \le 
C_1 + C_2\frac{\|f\|_\H^\alpha}{n^\beta}, ~~~C_1,C_2,\alpha,\beta \ge 0
\end{split}
\end{equation*}
Note that the regularization bounds, such as those for Tikhonov regularization, are also of similar form  as the choice of the regularization parameter implies an upper bound on the RKHS norm.  We see that our super-polynomial lower bound  on the norm $\|f\|_\H$ in Theorem~\ref{th:norm} implies that the right hand of this inequality diverges to infinity for any overfitted classifiers, making the bound trivial. There are some bounds logarithmic in the norm, such as the bound for the fat shattering in~\cite{belkin2018approximation} (used above) and  eigenvalue-dependent  bounds, which are potentially logarithmic, e.g., Theorem 13 of~\cite{goel2017eigenvalue}. However, as all of these bounds include a non-zero accuracy parameter, they do not apply to interpolated classifiers. 
Moreover, to account for the experiments with high label noise (below), any potential bound must have tight constants.  We do not know of any complexity-based bounds with this property. It is not clear such bounds exist.


\subsection{Experimental validation}
\noindent{\bf Zero label noise?}
A potential explanation for the disparity between  the consequences of lower norm bound in Theorem~\ref{th:norm} for classical generalization results and the performance observed in actual data, is the possibility that the error rate of the Bayes optimal classifier  (the ``label noise'') is is zero (e.g.,~\cite{soudry2017implicit}). 
Since our analysis relies on
$\expectp{P}{l(f(\bx),y)}>0$, the lower bound  in Eq.~\ref{eq:min_norm} does not hold if $y$ is a deterministic function\footnote{Note that even when $y$ is a deterministic function of $x$, the norm of the interpolated solution will diverge to infinity unless $y(\bx) \in \H$. Since $y(\bx)$ for classification is discontinuous, $y(\bx)$ is never in RKHS for smooth kernels. 
However, in this case, the growth of the norm of the interpolant as a function of $n$ requires other techniques to analyze. } of $\bx$.
Indeed, many classical bounds are available for ``overfitted'' classifiers  under zero label noise condition.  
For example, if  two classes are linearly separable, the classical  bounds (including those for the Perceptron algorithm) apply  to  linear classifiers with zero loss.
To resolve this issue, we provide experimental results demonstrating that near-optimal performance for overfitted kernel classifiers persists even for high levels of label noise. Thus, while  classical bounds may describe  generalization in zero noise regimes, they  cannot explain performance in noisy regimes. 
 We provide several lines of evidence: 
\begin{enumerate}
\item We study synthetic datasets, where the noise level is known a priori, showing that overfitted and interpolated classifiers consistently achieve error close to that of the Bayes optimal classifier, even for high noise levels.
\item By adding label noise to  real-world datasets we can guarantee non-zero Bayes risk. However, as we will see, performance of overfitted/interpolated kernel methods decays at or below the noise level, as it would for the Bayes optimal classifier.
\item We show that (as expected) for ``low noise'' synthetic and real datasets, adding 
small amounts of label noise leads to dramatic increases in the norms of overfitted solutions but only slight decreases in accuracy. For ``high noise'' datasets, adding label noise makes little difference for the norm but a similar decrease in classifier accuracy, consistent with the noise level.
\end{enumerate}
We first need the following (easily proved) proposition.
\begin{Proposition}{\it
 Let $P$ be a multiclass probability distribution  on $\Omega\times \{1,\ldots,k\}$. Let $P_\epsilon$ be the same distribution with the $\epsilon$ fraction of the labels flipped at random with equal probability. Then the following holds:\\
1. The Bayes optimal classifier $c^*$ for $P_\epsilon$ is the same as the Bayes optimal classifier for $P$. \\
2. The error rate ($0-1$ loss) 
\vskip-5pt
\begin{equation} \label{eq:error_opt}
{P_{\epsilon}} ({c^*(\bx) \ne y}) = \epsilon\frac{k-1}{k} + (1-\epsilon) P (c^*(\bx) \ne y)
\end{equation}

}
\end{Proposition}
\vskip-5pt
\noindent{\bf Remark.} Note that adding label noise to a probability distribution increases the error rate of the optimal classifier by at most $\epsilon$. In particular, when $k=2$ and $P$ has no label noise, the Bayes risk of $P_\epsilon$ is simply $\epsilon/2$. In addition, the loss of the Bayes optimal classifier is linear in $\epsilon$. 


\noindent {\bf A note on the experimental setting.}  In the experimental results  in this section we only use (smooth) Gaussian kernels to provide a setting consistent with Theorem~\ref{th:norm}. Overfitted classifiers are trained to have zero classification error using EigenPro\footnote{We  stop iteration when classification error reaches zero.}.  Interpolated classifiers are constructed by solving Eq.~\ref{eq:interpolant} directly\footnote{As interpolated classifiers are constructed by solving a  poorly conditioned system of equation,  the reported norm should be taken as  a lower bound on the actual norm.}.  

\begin{wrapfigure}{r}{0.3\textwidth}
  \centering
  \vspace{-5mm}
  \begin{minipage}[b]{.3\textwidth}
    \includegraphics[width=\textwidth]{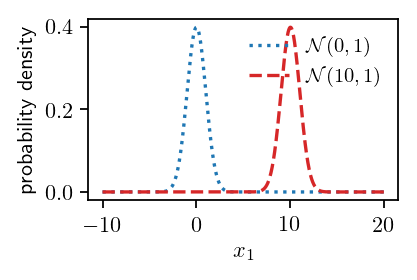}
  \end{minipage}
  \vspace{-10mm}
  \label{fig:ndist-sep}
\end{wrapfigure}
\noindent {\bf Synthetic dataset 1: Separable classes+noise.}
We start by considering a synthetic dataset in $\R^{50}$.  
Each data point $(\bx, y)$ is sampled as follows:
randomly sample label $y$ from $\{-1, 1 \}$ with equal probability;
for a given $y$, draw the first coordinate of  $\bx = (x_1, \ldots, x_{50}) \in \mathbb{R}^d$ 
from a univariate normal distribution conditional on the label and the rest uniformly from $[-1,1]$:
\begin{equation}\label{eq:synth-dist}
\begin{split}
x_1 &\sim \left\{\begin{matrix}
\mathcal{N}(0, 1), & \textnormal{if  } y = 1\\ 
\mathcal{N}(10, 1), & \textnormal{~~~if  } y = -1
\end{matrix}\right.  
~~~~~~x_2 \sim {U}(-1, 1), \ldots, x_{50} \sim {U}(-1, 1)
\end{split}
\end{equation}

\begin{wrapfigure}{r}{0.6\textwidth}
  \vspace{-8mm}
  \centering
  \includegraphics[width=0.6\textwidth]{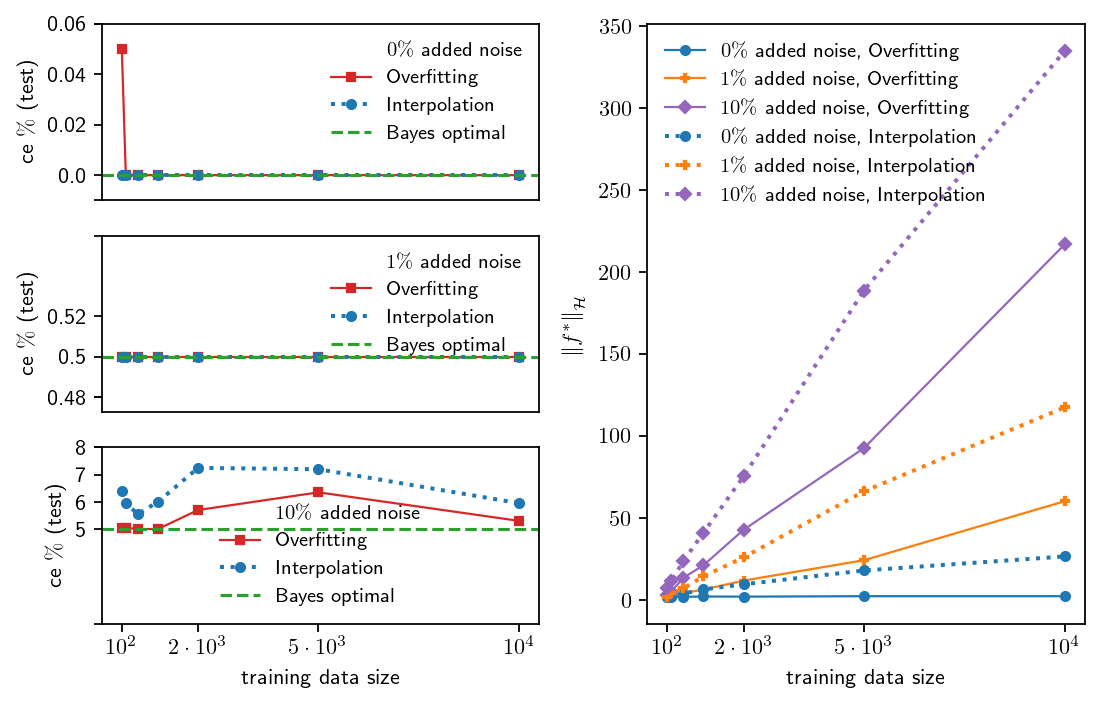}
  \vspace{-9mm}
  \caption{Overfitted/interpolated Gaussian kernel classifiers. Synthetic dataset 1. Left top to bottom: test error $0\%$, $1\%$, $10\%$ added label noise. Right: RKHS norms.}  
  \label{fig:synth-0-10}
  \vspace{-5mm}
\end{wrapfigure}
We see that the classes are (effectively) linearly separable, with the Bayes optimal classifier defined as
$c^*(\bx)=\sign (x_1 -5)$.

In Fig.~\ref{fig:synth-0-10}, we show  classification error rates for Gaussian kernel with a fixed kernel  parameter. We compare classifiers constructed to overfit the data by driving the classification error to zero iteratively (using EigenPro) to the direct numerical interpolating solution. We see that, as expected for linearly separable data, an overfitted solution achieves optimal 

\noindent accuracy with a small norm. The interpolated solution has a larger norm yet performs identically. 
On the other hand adding just $1\%$ label noise  increases the norm by more than an order of magnitude. However both overfitted and interpolated kernel classifiers still perform at $1\%$, the Bayes optimal level. Increasing the label noise to $10\%$ shows a similar  pattern, although the classifiers become slightly less accurate than the Bayes optimal. We see that there is little connection between the solution norm and the classifier performance. \newline
\indent Additionally, we observe that the norm of either solution increases quickly with the number of data points, a finding consistent with Theorem~\ref{th:norm}.\newline

\begin{wrapfigure}{r}{0.33\textwidth}
  \centering
  \vspace{-6mm}
  \begin{minipage}[b]{.3\textwidth}
    \includegraphics[width=\textwidth]{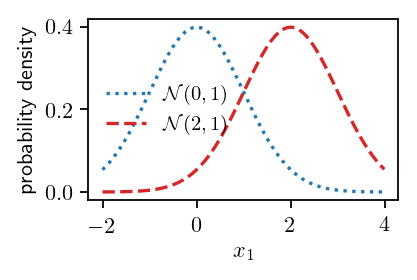}
  \end{minipage}
  \vspace{-1.6cm}
  \label{fig:ndist-nsep}
\end{wrapfigure}
\noindent {\bf Synthetic dataset 2: Non-separable classes.}
Consider the same setting as above, except that the Gaussian classes are moved within two standard deviations of each other
(right figure).
The classes are no longer separable, with the optimal classifier error of approximately $15.9\%$.\newline

 \begin{wrapfigure}{r}{0.6\textwidth}
  \centering
  \begin{minipage}[b]{0.6\textwidth}
   \vspace{-5mm}
  \centering
  \includegraphics[width=\textwidth]{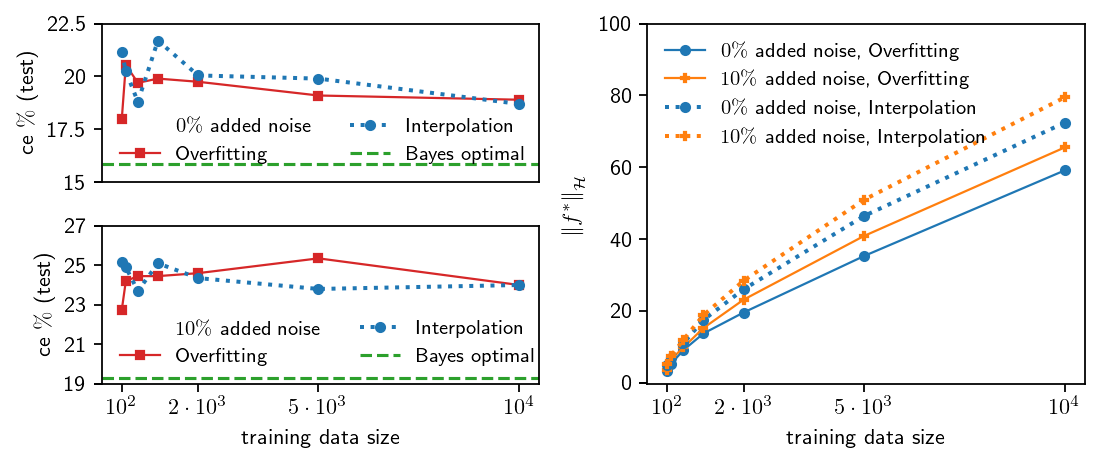}
  \vspace{-8mm}
  \caption{Overfitted and interpolated Gaussian classifiers for non-separable synthetic dataset with added label noise. Left: test error, Right: RKHS norms.}
  \label{fig:synth-0-2}
  \end{minipage}
 \end{wrapfigure}
Since the setting is already noisy, we expect that adding additional label noise should have little effect on the norm. This, indeed, is the case: See Fig~\ref{fig:synth-0-2} (bottom left panel).
We note that the accuracy of the interpolated classifier is  consistently within $5\%$ of the Bayes optimal, even with the added label noise.\newline\newline

\begin{wrapfigure}{r}{0.6\textwidth}
  \centering
  \vspace{-5mm}
  \begin{minipage}[b]{.6\textwidth}
    \includegraphics[width=\textwidth]{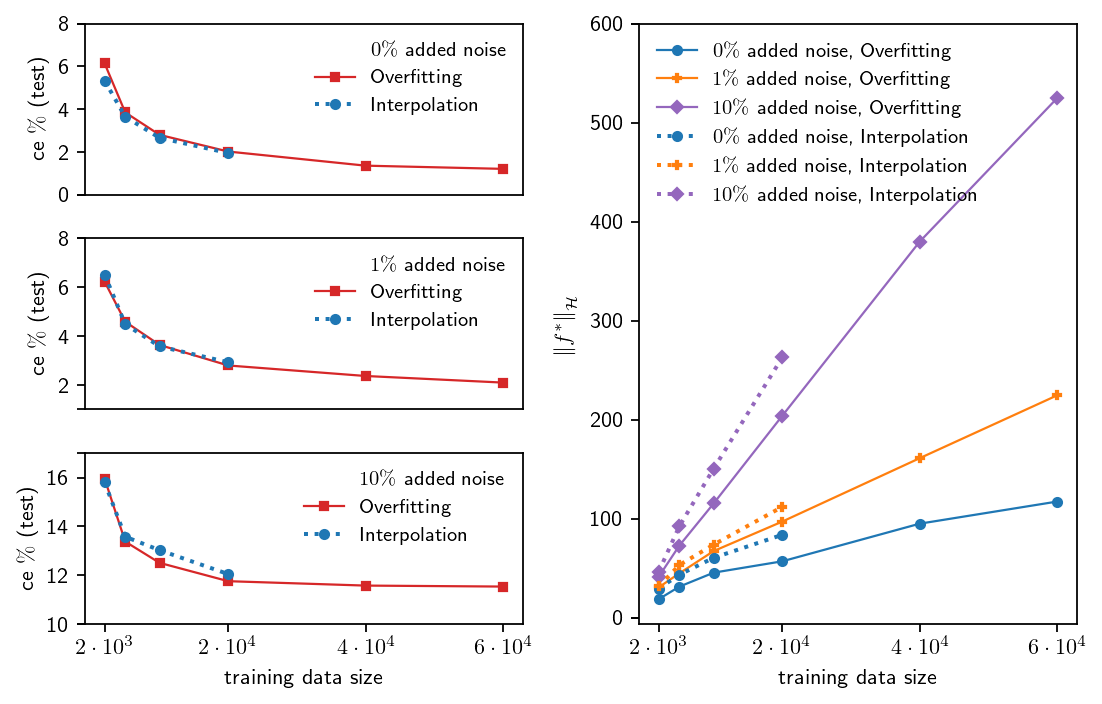}
    \vspace{-8mm}
    \subcaption{MNIST}
    \label{fig:lc-mnist}
  \end{minipage}
  
  \begin{minipage}[b]{.6\textwidth}
    \includegraphics[width=\textwidth]{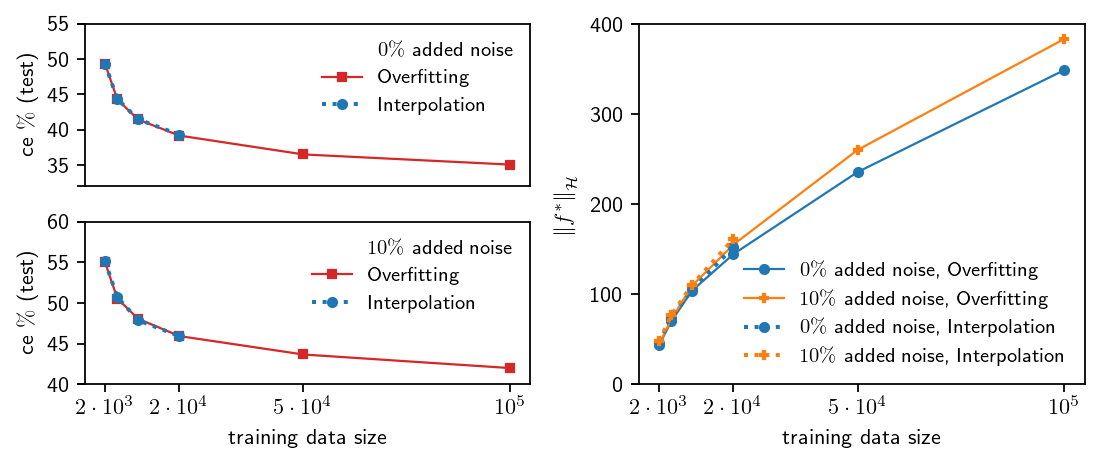}
    \vspace{-8mm}
    \subcaption{TIMIT}
    \label{fig:lc-timit}
  \end{minipage}
  \vspace{-8mm}
  \caption{Overfitted/interpolated Gaussian classifiers (MNIST/TIMIT), added label noise. Left: test error, Right: RKHS norms.}
  \label{fig:lc-real}
  \vspace{-5mm}
\end{wrapfigure}
\noindent {\bf Real data + noise.}
We consider two real-data multiclass datasets (MNIST and TIMIT).  MNIST labels
are arguably close to a deterministic  function of the features, as most (but not all) digit images are easily recognizable. On the other hand, phonetic classification task in TIMIT seem to be significantly more uncertain and inherently noisier. 

This is reflected in the state-of-the-art error rates, less than $0.3\%$ for (10-class) MNIST~\cite{wan2013regularization} and over $30\%$ for (48-class) TIMIT \cite{may2017kernel}.  
While the true Bayes risk for real data cannot be ascertained, we can ensure that it is non-zero by adding label noise.  

Consistently with the expectations,  adding even $1\%$ label noise significantly increases the norm of overfitted/interpolated solutions norm for ``clean''  MNIST, while even additional $10\%$ noise makes only marginal difference for ``noisy'' TIMIT (Fig.~\ref{fig:lc-real}). 
On the other hand, the test performance on either dataset decays gracefully with the amount of noise, as it would for optimal classifiers (according to Eq.~\ref{eq:error_opt}).\newline

\begin{wrapfigure}{r}{0.61\textwidth}
  \centering
  \vspace{-8mm}

  \begin{minipage}[l]{0.3\textwidth}
    \includegraphics[width=\textwidth]{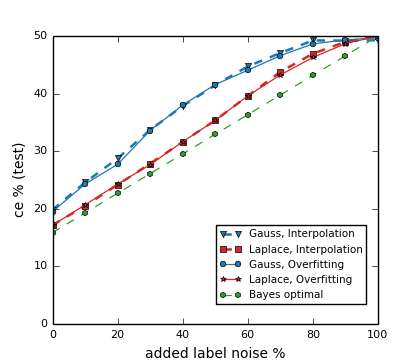}
    \subcaption{Synthetic-2}
  \end{minipage}
  \hfill 
  \begin{minipage}[l]{0.3\textwidth}
    \includegraphics[width=\textwidth]{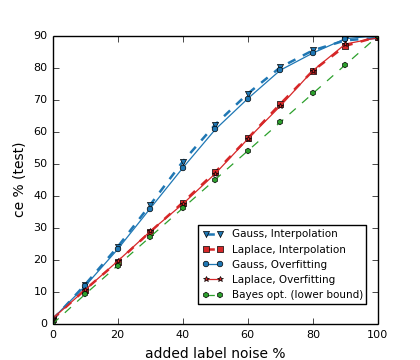}
    \subcaption{MNIST}
    
  \end{minipage}
\vspace{-2mm}
\caption{
Overfitted classifiers, interpolated classifiers, and Bayes optimal for datasets with added label noise. y axis: classification error on test.
}
  \label{fig:ce_test-lc-o-i-b}
\end{wrapfigure}
\noindent {\bf High label noise Bayes risk comparison.} In Fig.~\ref{fig:ce_test-lc-o-i-b} we show performance of Gaussian and Laplacian kernels for different levels of added label noise for Synthetic-2 and MNIST datasets. We see that interpolated kernel classifiers perform well and closely track the Bayes risk\footnote{As we do not know the true Bayes risk for MNIST, we use a lower bound by simply assuming it is zero. The ``true'' Bayes risk is likely slightly higher than our curve. } even for 
very high levels of label noise.  There is minimal deterioration as the level of label noise increases. Even at $80\%$ label corruption they perform well above chance. 
Consistently with our observations above, there is very little difference in performance between interpolated and overfitted classifiers. This graph illustrates the difficulty of constructing a non-trivial generalization bound for these noisy regimes, which would have to provide values in the narrow band between the Bayes risk and the risk of a random guess.\\



\section{Fitting noise: Laplacian and  Gaussian kernels, connections to ReLU Networks}
\label{sec:fit-noise}

\noindent 
\begin{wraptable}{r}{6.5cm}
\centering
\vspace{-5mm}
\caption{Epochs to overfit (Laplacian)}
\label{tbl:epoch-ep-lap}
\vspace{-2mm}
\resizebox{6.5cm}{!}{
\begin{tabular}{|c||c|c|c|}
\hline
Label & MNIST & SVHN & TIMIT \\ \hline
Original & 4 & 8 & 3 \\ \hline
Random & 7 & 21 & 4 \\ \hline
\end{tabular}
}

\vspace{2mm}
\caption{Epochs to overfit (Gaussian)}
\label{tbl:epoch-ep-gau}
\vspace{-3mm}
\resizebox{6.5cm}{!}{
\begin{tabular}{|c||c|c|c|}
\hline
Label & MNIST & SVHN & TIMIT \\ \hline
Original & 20 & 46 & 7 \\ \hline
Random & 873 & 1066 & 22 \\ \hline
\end{tabular}
}
\end{wraptable}
{\bf Laplacian kernels and ReLU networks.}
We will now point out some interesting similarities between  Laplacian kernel machines and ReLU networks.
In~\cite{zhang2016understanding} the authors showed that
ReLU neural networks are easily capable of 
fitting labels randomly assigned to the original features, needing only about three times as many iterations of SGD as for the original labels.
In Table~\ref{tbl:epoch-ep-lap} we demonstrate a very similar finding for Laplacian kernels. We see that the number of epochs needed to fit random labels is no more than twice  that for the original labels. Thus,
SGD-type methods with Laplacian kernel have very high computational reach, similar to that of ReLU networks. We note that Laplacian kernels are non-smooth, with a discontinuity of the derivative reminiscent of that for ReLU units. We conjecture that optimization performance is controlled by the type of non-smoothness.

\noindent
{\bf Laplacian and Gaussian kernels.}
On the other hand, training Gaussian kernels to fit noise is far more computationally intensive (see Table.~\ref{tbl:epoch-ep-gau}), as suggested by the bounds on fat shattering dimension for smooth kernels~\cite{belkin2018approximation}. As we see from the table, Gaussian kernels also require  many more epochs to fit the original labels. 
On the other hand, overfitted/interpolated  Gaussian and Laplacian kernels show very similar classification and regression performance on test data (Section~\ref{sec:interpolated}). That similarity persists even with added label noise, see Fig.~\ref{fig:noise_gaus_lap}. 
Hence it appears that
the generalization properties of these classifiers are not related to the specifics of the optimization.
\begin{wrapfigure}{r}{0.6\textwidth}
\vspace{-1mm}
  \centering
  \begin{minipage}[b]{.295\textwidth}
    \includegraphics[width=\textwidth]{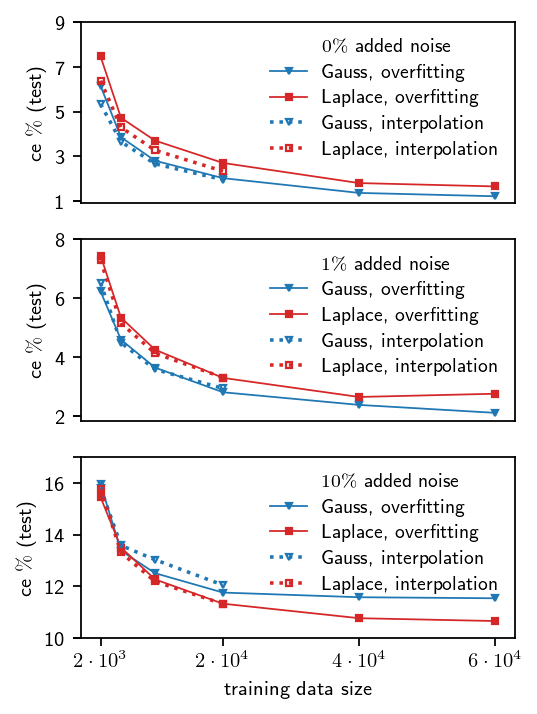}
    \vspace{-5mm}
    \subcaption{MNIST}
  \end{minipage}
  \hfill
  \begin{minipage}[b]{.295\textwidth}
    \includegraphics[width=\textwidth]{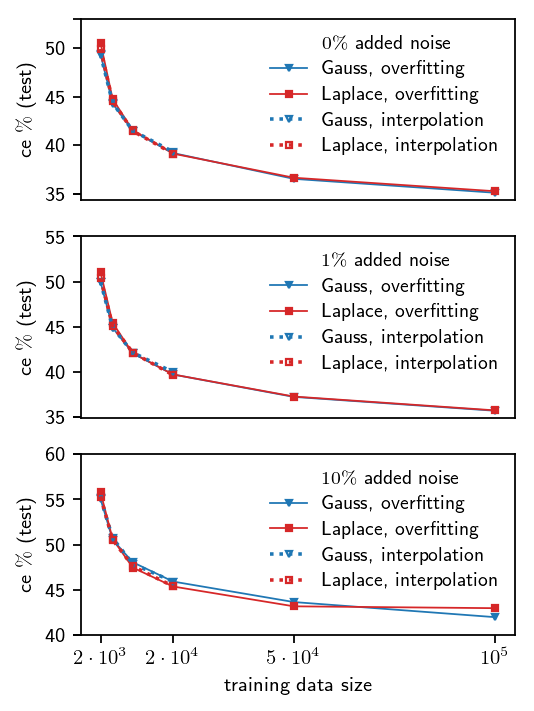}
    \vspace{-5mm}
    \subcaption{TIMIT}
  \end{minipage}
  \vspace{-5mm}
  \caption{Overfitted and interpolated classifiers using Gaussian kernel and Laplace kernel for datasets with added label noise (top: $0\%$, middle: $1\%$, bottom: $10\%$)}
  \label{fig:noise_gaus_lap}
\end{wrapfigure}
We conjecture that the radial structure of these two kernels plays a key role in ensuring strong classification performance.

\noindent{\bf A note on computational efficiency.}
In our experiments EigenPro traced a very similar optimization path to SGD/Pegasos while providing 10X-40X acceleration in terms of the number of epochs (with about 15\% overhead). 
When combined with Laplacian kernels,  optimal performance is consistently achieved in under $10$ epochs. 
We believe that  methods using  Laplacian kernels hold significant promise for future work on scaling to very large data.

\section*{Acknowledgements}
We  thank Raef Bassily, Daniel Hsu and Partha Mitra for numerous discussions, insightful questions and comments. We thank Like Hui for preprocessing the 20 Newsgroups dataset.
We used a Titan Xp GPU provided by Nvidia. We are grateful to NSF for financial support.

\clearpage
\bibliographystyle{alpha}
\bibliography{ref}

\clearpage
\begin{appendices}
\section{Experimental Setup}
\label{app:expr-setup}
\noindent {\bf Computing Resource.} All experiments were run on a single workstation equipped with 128GB main memory, two Intel Xeon(R) E5-2620 processors, and one Nvidia GTX Titan Xp (Pascal) GPU.

\begin{wraptable}{r}{8.5cm}
\centering
\vspace{-4mm}
\begin{adjustbox}{center}
\resizebox{8.5cm}{!}{
\begin{tabular}{|l||c|c|c|}
\hline
Dataset & $n$ & $d$ & Label \\ \hline 
CIFAR-10
& $5 \times 10^4$ & 1024 & \{0,...,9\} \\ \hline 
MNIST
& $6 \times 10^4$ & 784 & \{0,...,9\} \\ \hline 
SVHN
& $7 \times 10^4$ & 1024 & \{1,...,10\} \\ \hline 
HINT-S
& $2 \times 10^5$ & 425 & $\{0,1\}^{64}$ \\ \hline 
TIMIT
& $1.1 \times 10^6$ & 440 & \{0,...,143\} \\ \hline 
20 Newsgroups
& $1.6 \times 10^4$ & 100 & $\{0,1\}^{20}$ \\ \hline 
\end{tabular}
}
\end{adjustbox}
\vspace{-5mm}
\end{wraptable}
\noindent {\bf Datasets.}
The table on the right summarizes
the datasets used in experiments.
We map multiclass labels to multiple binary labels (e.g. one label of $c$ classes to one $c$-length binary vector).
For image datasets including MNIST~\cite{lecun1998gradient}, CIFAR-10~\cite{krizhevsky2009learning}, and SVHN~\cite{netzer2011reading}, color images are first transformed to grayscale images. We then rescale the range of each feature to $[0, 1]$.
For HINT-S~\cite{healy2013algorithm} and TIMIT~\cite{garofolo1993darpa}, we normalize each feature by z-score. 
To efficiently fit the
20 Newsgroups~\cite{Lang95} dataset
with kernel regression, we transform its sparse feature vector (bag of words) into dense feature vector by summing up the corresponding embeddings of the words from~\cite{pennington2014glove}.

\begin{wraptable}{r}{5.5cm}
\centering
\vspace{-4mm}
\begin{adjustbox}{center}
\resizebox{5.5cm}{!}{
\begin{tabular}{|l||c|c|}
\hline
\multicolumn{1}{|l||}{Dataset} & Gauss & Laplace\\ \hline
CIFAR-10 & 5 & 10 \\ \hline
MNIST & 5 & 10 \\ \hline
SVHN & 5 & 10 \\ \hline
HINT-S & 11 & 20 \\ \hline
TIMIT & 16 & 20 \\ \hline
20 News & 0.1 & 0.1 \\ \hline
\end{tabular}
}
\end{adjustbox}
\vspace{0mm}
\end{wraptable}
\noindent {\bf Hyperparameters.} For consistent comparison, all iterative methods use mini-batch of size $m = 256$.
The EigenPro preconditioner in~\cite{ma2017diving} is constructed using the top $k = 160$ eigenvectors of a subsampled training set of size $M = 5000$ (or the training set when its size is less than 5000).

\noindent {\bf Kernel Bandwidth Selection.} For each dataset, we select the bandwidth $\sigma$ for Gaussian kernel $k(x, y) = \exp(-\frac{\norm{x - y}^2}{2\sigma^2})$ and Laplace kernel $k(x, y) = \exp(-\frac{\norm{x - y}}{\sigma})$ through cross-validation on a small subsampled dataset. The final bandwidths used for all datasets are listed in the table on the right side.

\section{Detailed experimental results}
\label{app:exper-results} 

\begin{wrapfigure}{r}{0.71\textwidth}
  \centering
  \vspace{-8mm}

  \begin{minipage}[l]{0.35\textwidth}
    \includegraphics[width=\textwidth]{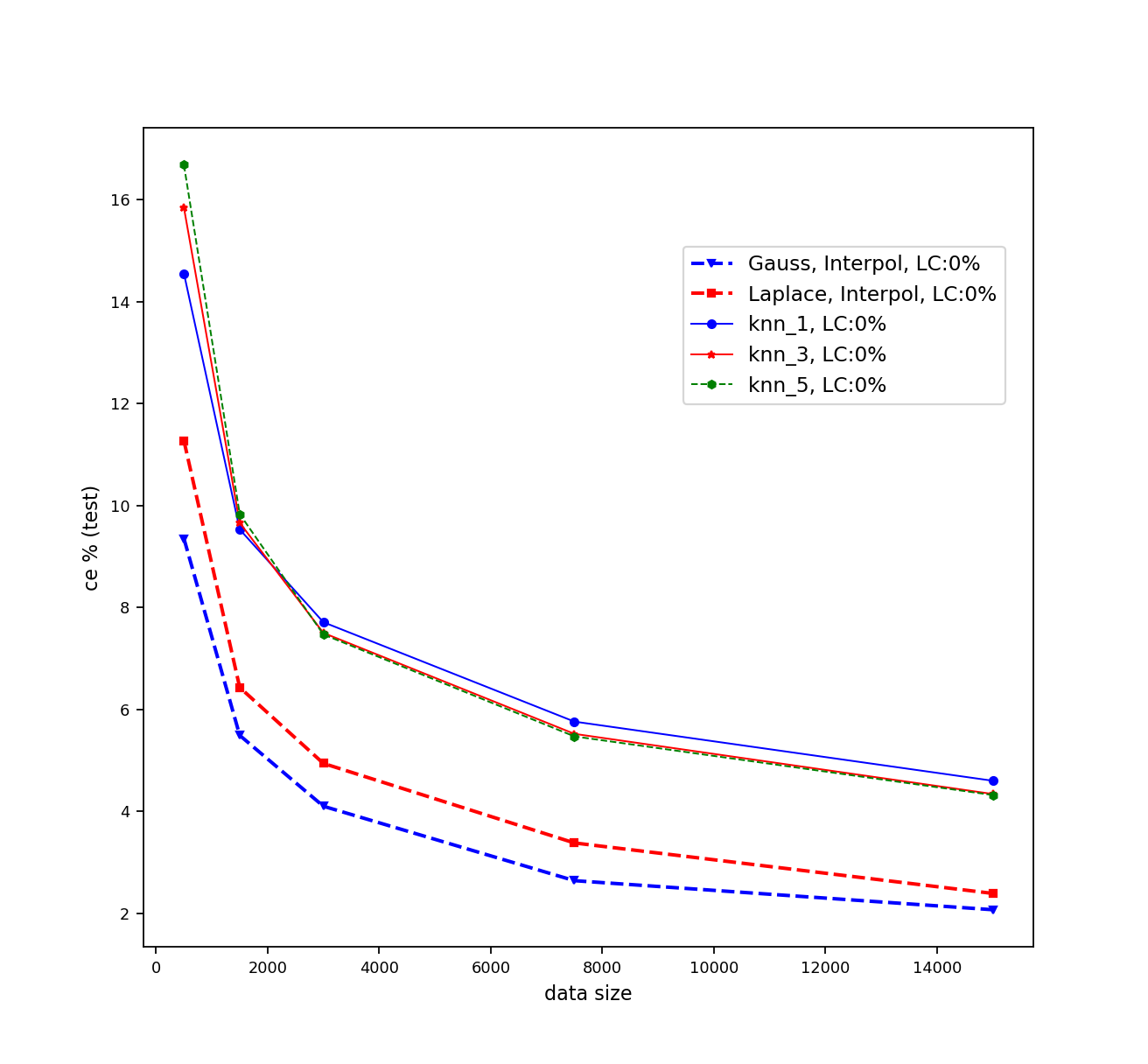}
    \subcaption{Label noise 0\%}
  \end{minipage}
  \hfill 
  \begin{minipage}[l]{0.35\textwidth}
    \includegraphics[width=\textwidth]{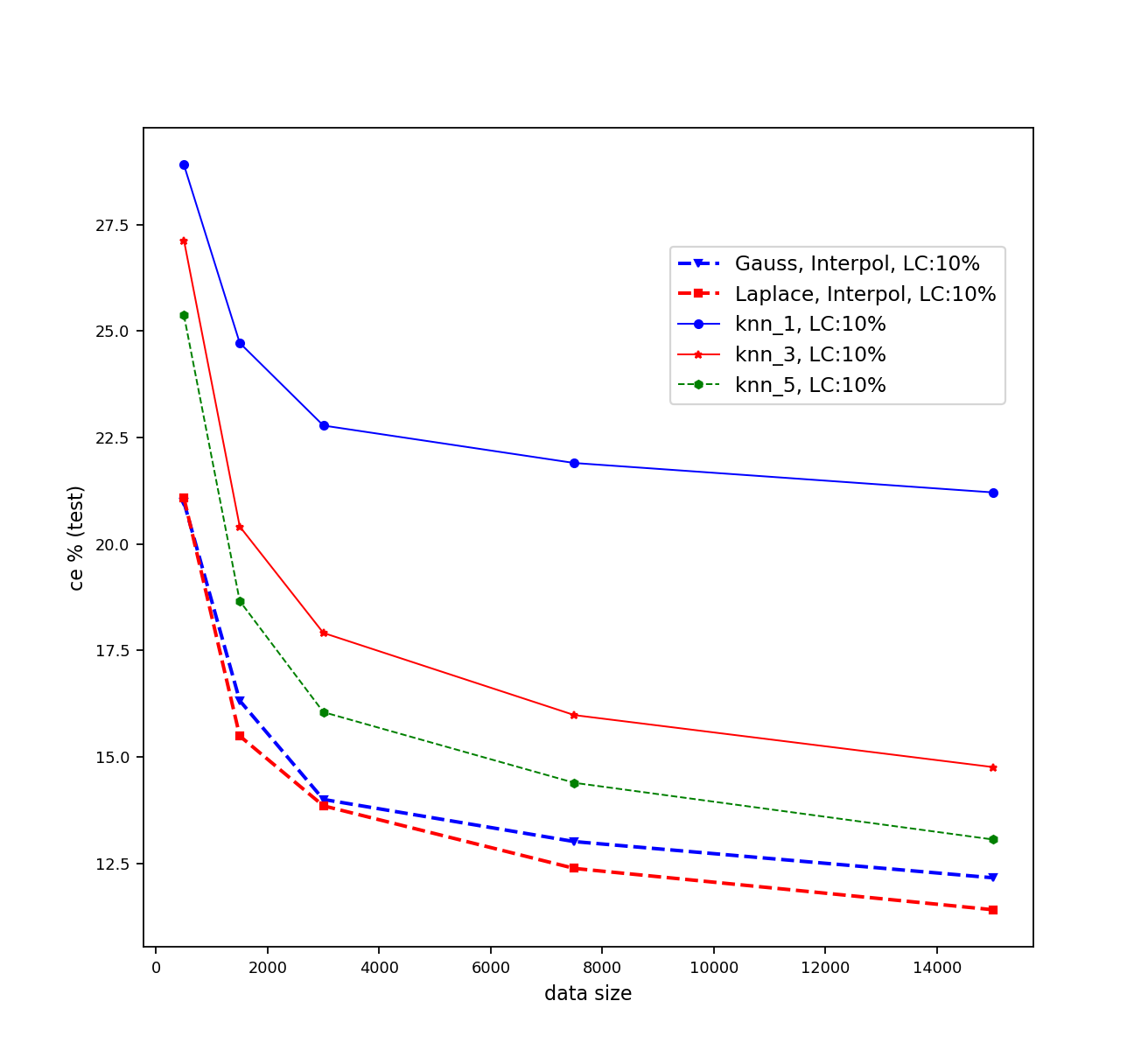}
    \subcaption{Label noise 10\%}
    
  \end{minipage}
\vspace{-2mm}
\caption{
Interpolated classifiers, and k-NN for MNIST dataset. Added label noise 0\% and 10\%. y axis: classification error on test data. x axis: training data size.
}
  \label{fig:ce_test-diffdatasize-i-knn-mnist}
\end{wrapfigure}
Below in Tables~\ref{mnist-cifar-10-full-table},\ref{svhn-timit-full-table},\ref{hint-s-20newsgroups-full-table} we provide exact detailed numerical  results for the graphs given in Section~\ref{sec:interpolated}. 
In Table~\ref{cetest-datasize-labelnoise0-mnist} and~\ref{cetest-datasize-labelnoise10-mnist} test classification errors have been compared among different training data size and different methods with 0\% and 10\% added label noise respectively. Fig.~\ref{fig:ce_test-diffdatasize-i-knn-mnist} shows this comparison.
Fig.~\ref{fig:ce_test-lc-o-i-b-diffbw} shows results with different bandwidths. Three settings for bandwidth have been considered: 50\%, 100\% and 200\% of the bandwidth selected for optimal performance (different for Gaussian and Laplacian kernels).

\begin{wrapfigure}{r}{0.81\textwidth}
  \centering
  \vspace{-3mm}

  \begin{minipage}[l]{0.4\textwidth}
    \includegraphics[width=\textwidth]{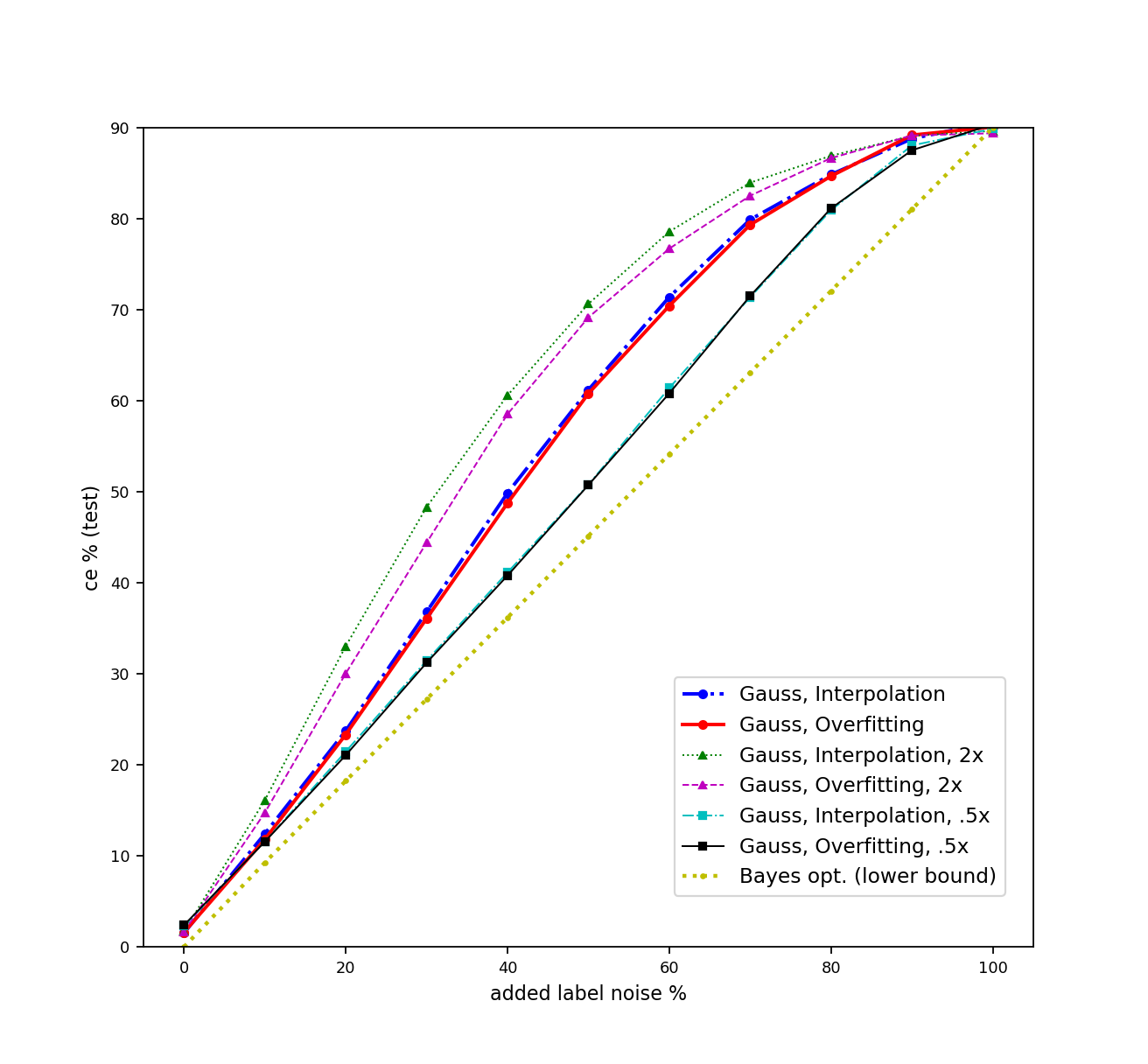}
    \subcaption{Gauss}
  \end{minipage}
  \hfill 
  \begin{minipage}[l]{0.4\textwidth}
    \includegraphics[width=\textwidth]{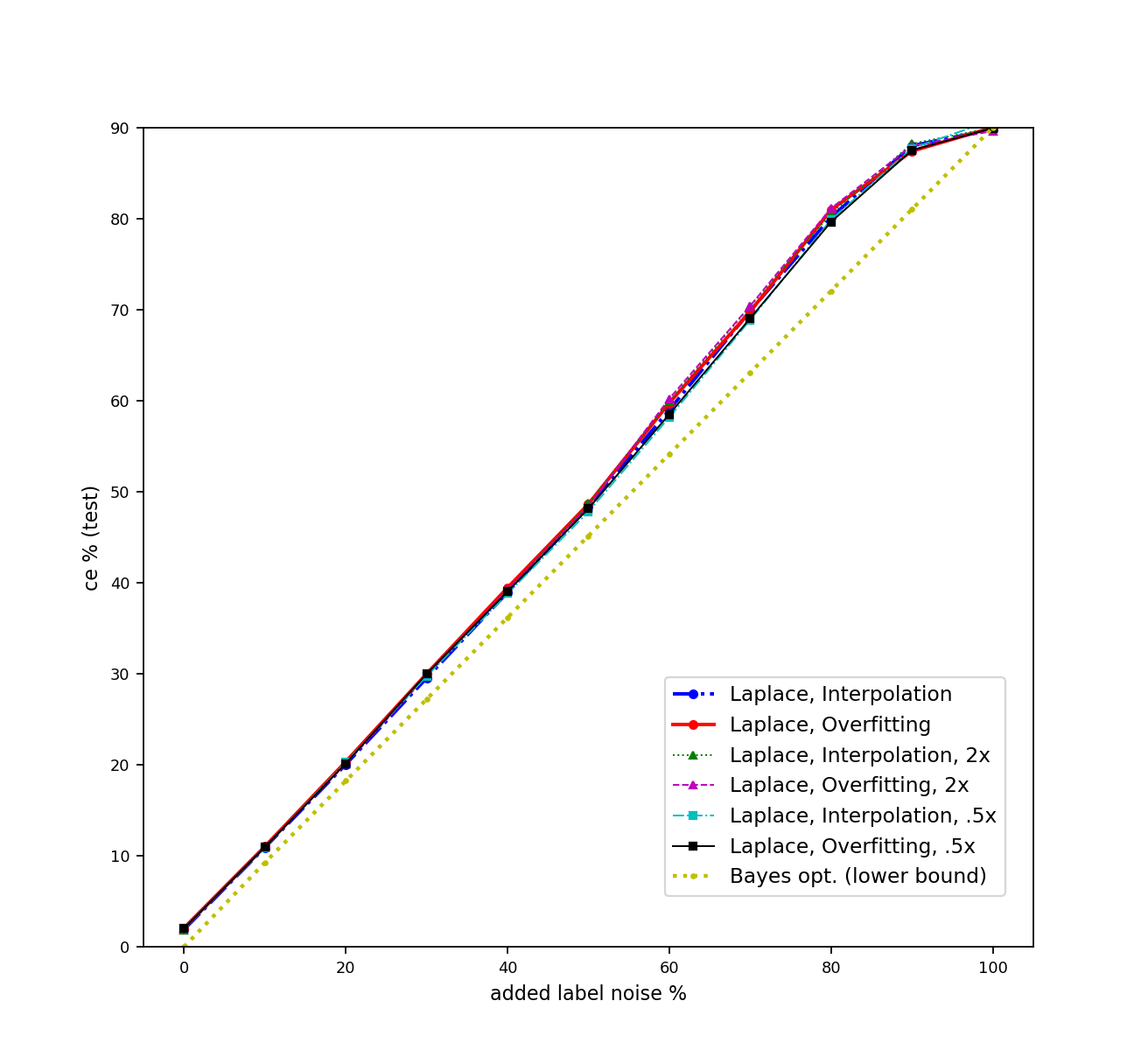}
    \subcaption{Laplace}
    
  \end{minipage}
\vspace{-2mm}
\caption{
Overfitted classifiers, interpolated classifiers, and Bayes optimal for MNIST datasets with added label noise with different kernel bandwidth. y axis: classification error on test data.
}
  \label{fig:ce_test-lc-o-i-b-diffbw}
\end{wrapfigure}


\begin{table}[]
\centering
\small
\resizebox{15cm}{!}{%
\begin{tabular}{|c|c||c|c|c|c|c||c|c|c|c|c|}
\hline
\multirow{2}{*}{}                                                           & \multirow{2}{*}{Kernel, Method} & \multicolumn{5}{c||}{Epochs (MNIST)}      & \multicolumn{5}{c|}{Epochs (CIFAR-10)}  \\ \cline{3-12} 
                                                                            &                                 & 1     & 2     & 5      & 10     & 20     & 1     & 2     & 5     & 10     & 20     \\ \hline
\hline 
\multirow{4}{*}{\begin{tabular}[c]{@{}c@{}}ce \\ \%\\ (test)\end{tabular}}  & Gauss, Eipro                    & 1.74  & 1.42  & 1.26   & 1.21   & 1.24   & 57.94 & 54.27 & 50.74 & 50.88  & 51.17  \\ \cline{2-12} 
                                                                            & Laplace, Eipro                  & 2.13  & 1.73  & 1.61   & 1.57   & 1.58   & 55.86 & 51.46 & 48.98 & 48.8   & 48.79  \\ \cline{2-12} 
                                                                            & Gauss, Interp                   & \multicolumn{5}{c||}{1.24}                & \multicolumn{5}{c|}{51.56}              \\ \cline{2-12} 
                                                                            & Laplace, Interp                 & \multicolumn{5}{c||}{1.57}                & \multicolumn{5}{c|}{48.75}              \\ \hline
\hline 
\multirow{4}{*}{\begin{tabular}[c]{@{}c@{}}ce \\ \%\\ (train)\end{tabular}} & Gauss, Eipro                    & 0.44  & 0.16  & 0.018  & 0.003  & 0.0    & 18.73 & 5.28  & 0.23  & 0.03   & 0.0    \\ \cline{2-12} 
                                                                            & Laplace, Eipro                  & 0.32  & 0.03  & 0.0    & 0.0    & 0.0    & 6.95  & 0.23  & 0.0   & 0.0    & 0.0    \\ \cline{2-12} 
                                                                            & Gauss, Interp                   & \multicolumn{5}{c||}{0}                   & \multicolumn{5}{c|}{0}                  \\ \cline{2-12} 
                                                                            & Laplace, Interp                 & \multicolumn{5}{c||}{0}                   & \multicolumn{5}{c|}{0}                  \\ \hline
\hline 
\hline 
\multirow{4}{*}{\begin{tabular}[c]{@{}c@{}}mse \\ (test)\end{tabular}}      & Gauss, Eipro                    & 0.077 & 0.066 & 0.06   & 0.05   & 0.05   & 3.42  & 3.14  & 2.92  & 2.91   & 2.95   \\ \cline{2-12} 
                                                                            & Laplace, Eipro                  & 0.083 & 0.07  & 0.062  & 0.06   & 0.06   & 2.94  & 2.77  & 2.67  & 2.65   & 2.65   \\ \cline{2-12} 
                                                                            & Gauss, Interp                   & \multicolumn{5}{c||}{0.05}                & \multicolumn{5}{c|}{3.00}               \\ \cline{2-12} 
                                                                            & Laplace, Interp                 & \multicolumn{5}{c||}{0.06}                & \multicolumn{5}{c|}{2.65}               \\ \hline
\hline 
\multirow{4}{*}{\begin{tabular}[c]{@{}c@{}}mse \\ (train)\end{tabular}}     & Gauss, Eipro                    & 0.049 & 0.031 & 0.012  & 0.005  & 0.002  & 1.88  & 1.07  & 0.32  & 0.09   & 0.02   \\ \cline{2-12} 
                                                                            & Laplace, Eipro                  & 0.046 & 0.022 & 3.9e-3 & 3.7e-4 & 8.2e-6 & 1.44  & 0.69  & 0.09  & 5.4e-3 & 4.4e-5 \\ \cline{2-12} 
                                                                            & Gauss, Interp                   & \multicolumn{5}{c||}{3.2e-27}             & \multicolumn{5}{c|}{1.5e-8}             \\ \cline{2-12} 
                                                                            & Laplace, Interp                 & \multicolumn{5}{c||}{4.6e-28}             & \multicolumn{5}{c|}{1.6e-8}             \\ \hline
\end{tabular}
}
\caption{MNIST and CIFAR-10 summary table.}
\label{mnist-cifar-10-full-table}
\end{table}

\begin{table}[]
\label{table:summary1}
\centering 
\small 
\resizebox{15cm}{!}{%
\begin{tabular}{|c|c||c|c|c|c|c||c|c|c|c|c|}
\hline
\multirow{2}{*}{}                                                           & \multirow{2}{*}{Kernel, Method} & \multicolumn{5}{c||}{Epochs (SVHN 20k)}  & \multicolumn{5}{c|}{Epochs (TIMIT 50k)}  \\ \cline{3-12} 
                                                                            &                                 & 1     & 2     & 5     & 10     & 20     & 1     & 2     & 5      & 10     & 20     \\ \hline
                                                                            \hline 
\multirow{4}{*}{\begin{tabular}[c]{@{}c@{}}ce \\ \%\\ (test)\end{tabular}}  & Gauss, Eipro                    & 33.57 & 29.32 & 25.37 & 23.82  & 23.96  & 40.04 & 38.16 & 36.63  & 36.56  & 36.57  \\ \cline{2-12} 
                                                                            & Laplace, Eipro                  & 30.47 & 27.48 & 24.27 & 23.47  & 23.51  & 38.36 & 36.95 & 36.5   & 36.5   & 36.51  \\ \cline{2-12} 
                                                                            & Gauss, Interp                   & \multicolumn{5}{c||}{24.30}              & \multicolumn{5}{c|}{36.61}               \\ \cline{2-12} 
                                                                            & Laplace, Interp                 & \multicolumn{5}{c||}{23.54}              & \multicolumn{5}{c|}{36.51}               \\ \hline
\hline                                                                             
\multirow{4}{*}{\begin{tabular}[c]{@{}c@{}}ce \\ \%\\ (train)\end{tabular}} & Gauss, Eipro                    & 10.81 & 4.19  & 0.71  & 0.07   & 0.0    & 2.13  & 0.28  & 0.0    & 0.002  & 0.0    \\ \cline{2-12} 
                                                                            & Laplace, Eipro                  & 3.79  & 0.34  & 0.0   & 0.0    & 0.0    & 0.39  & 0.0   & 0.0    & 0.0    & 0.0    \\ \cline{2-12} 
                                                                            & Gauss, Interp                   & \multicolumn{5}{c||}{0}                  & \multicolumn{5}{c|}{0}                   \\ \cline{2-12} 
                                                                            & Laplace, Interp                 & \multicolumn{5}{c||}{0}                  & \multicolumn{5}{c|}{0}                   \\ \hline
\hline                                                                             
\hline 
\multirow{4}{*}{\begin{tabular}[c]{@{}c@{}}mse \\ (test)\end{tabular}}      & Gauss, Eipro                    & 2.69  & 2.30  & 1.96  & 1.87   & 1.88   & 0.75 & 0.75  & 0.74   & 0.74   & 0.74   \\ \cline{2-12} 
                                                                            & Laplace, Eipro                  & 2.09  & 1.89  & 1.76  & 1.73   & 1.73   & 0.75 & 0.75 & 0.75  & 0.75  & 0.75  \\ \cline{2-12} 
                                                                            & Gauss, Interp                   & \multicolumn{5}{c||}{1.89}               & \multicolumn{5}{c|}{0.73}                \\ \cline{2-12} 
                                                                            & Laplace, Interp                 & \multicolumn{5}{c||}{1.73}               & \multicolumn{5}{c|}{0.73}                \\ \hline
\hline                                                                             
\multirow{4}{*}{\begin{tabular}[c]{@{}c@{}}mse \\ (train)\end{tabular}}     & Gauss, Eipro                    & 1.71  & 0.95  & 0.34  & 0.08   & 0.01   & 0.163 & 0.065 & 0.006  & 5.8e-4 & 3.1e-5 \\ \cline{2-12} 
                                                                            & Laplace, Eipro                  & 1.06  & 0.52  & 0.08  & 4.5e-3 & 4.2e-5 & 0.059 & 0.015 & 4.5e-4 & 4.8e-6 & 2.1e-7 \\ \cline{2-12} 
                                                                            & Gauss, Interp                   & \multicolumn{5}{c||}{8.9e-27}            & \multicolumn{5}{c|}{2.5e-10}             \\ \cline{2-12} 
                                                                            & Laplace, Interp                 & \multicolumn{5}{c||}{2.0e-26}            & \multicolumn{5}{c|}{1.1e-9}              \\ \hline
\end{tabular}
}
\caption{ SVHN and TIMIT summary table.}
\label{svhn-timit-full-table}
\end{table}

\begin{table}[]
\label{table:summary1}
\centering
\small 
\resizebox{15cm}{!}{%
\begin{tabular}{|c|c||c|c|c|c|c||c|c|c|c|c|}
\hline
\multirow{2}{*}{}                                                           & \multirow{2}{*}{Kernel, Method} & \multicolumn{5}{c||}{Epochs (HINT-S-20k)} & \multicolumn{5}{c|}{Epochs (20 Newsgroups)} \\ \cline{3-12} 
                                                                            &                                 & 1      & 2      & 5     & 10    & 20     & 1       & 2      & 5      & 10     & 20     \\ \hline
\hline 
\multirow{4}{*}{\begin{tabular}[c]{@{}c@{}}ce \\ \%\\ (test)\end{tabular}}  & Gauss, Eipro                    & 15.71  & 14.55  & 13.67 & 13.28 & 13.10  & 52.95   & 49.55  & 41.75  & 40.20  & 36.15  \\ \cline{2-12} 
                                                                            & Laplace, Eipro                  & 15.09  & 13.67  & 12.76 & 12.68 & 12.67  & 49.95   & 44.65  & 35.45  & 34.59  & 34.20  \\ \cline{2-12} 
                                                                            & Gauss, Interp                   & \multicolumn{5}{c||}{13.67}               & \multicolumn{5}{c|}{38.75}                  \\ \cline{2-12} 
                                                                            & Laplace, Interp                 & \multicolumn{5}{c||}{12.65}               & \multicolumn{5}{c|}{33.95}                  \\ \hline
\hline 
\multirow{4}{*}{\begin{tabular}[c]{@{}c@{}}ce \\ \%\\ (train)\end{tabular}} & Gauss, Eipro                    & 10.99  & 7.60   & 3.69  & 1.56  & 0.37   & 30.82   & 24.52  & 11.71  & 10.08  & 2.34   \\ \cline{2-12} 
                                                                            & Laplace, Eipro                  & 5.94   & 1.13   & 0.02  & 0.0   & 0.0    & 10.35  & 2.29  & 0.11  & 0.05  & 0.01  \\ \cline{2-12} 
                                                                            & Gauss, Interp                   & \multicolumn{5}{c||}{0}                   & \multicolumn{5}{c|}{0}                      \\ \cline{2-12} 
                                                                            & Laplace, Interp                 & \multicolumn{5}{c||}{0}                   & \multicolumn{5}{c|}{0}                      \\ \hline
\hline 
\hline 
\multirow{4}{*}{\begin{tabular}[c]{@{}c@{}}mse \\ (test)\end{tabular}}      & Gauss, Eipro                    & 7.91   & 7.35   & 6.90  & 6.72  & 6.61   & 1.30    & 1.23   & 0.84   & 0.79   & 0.69   \\ \cline{2-12} 
                                                                            & Laplace, Eipro                  & 7.56   & 6.76   & 6.22  & 6.11  & 6.09   & 1.15   & 0.87  & 0.61  & 0.59  & 0.59  \\ \cline{2-12} 
                                                                            & Gauss, Interp                   & \multicolumn{5}{c||}{7.26}                & \multicolumn{5}{c|}{0.96}                   \\ \cline{2-12} 
                                                                            & Laplace, Interp                 & \multicolumn{5}{c||}{6.09}                & \multicolumn{5}{c|}{0.59}                   \\ \hline
\hline 
\multirow{4}{*}{\begin{tabular}[c]{@{}c@{}}mse \\ (train)\end{tabular}}     & Gauss, Eipro                    & 5.87   & 4.49   & 2.82  & 1.67  & 0.78   & 0.99   & 0.79  & 0.32  & 0.22  & 0.09  \\ \cline{2-12} 
                                                                            & Laplace, Eipro                  & 3.87  & 1.82  & 0.31 & 0.027 & 7.7e-4 & 1.06   & 0.52  & 0.04  & 0.007  & 0.002  \\ \cline{2-12} 
                                                                            & Gauss, Interp                   & \multicolumn{5}{c||}{5.5e-7}              & \multicolumn{5}{c|}{2.3e-22}                \\ \cline{2-12} 
                                                                            & Laplace, Interp                 & \multicolumn{5}{c||}{5.7e-9}              & \multicolumn{5}{c|}{7.5e-28}                \\ \hline
\end{tabular}
}
\caption{ HINT-S-20k and 20 Newsgroups summary table.}
\label{hint-s-20newsgroups-full-table}
\end{table}


\begin{table}[]
\centering 
\begin{tabular}{|c|c|c|c|c|c|c|}
\hline
\multirow{2}{*}{added label noise = 0\%} & \multirow{2}{*}{Methods} & \multicolumn{5}{c|}{Data Size (MNIST)} \\ \cline{3-7} 
                                         &                          & 500    & 1500  & 3000  & 7500  & 15000 \\ \hline
\multirow{5}{*}{ce \% (test)}            & Gauss, Interpolation     & 9.35   & 5.49  & 4.10  & 2.64  & 2.07  \\ \cline{2-7} 
                                         & Laplace, Interploation   & 11.27  & 6.43  & 4.94  & 3.38  & 2.39  \\ \cline{2-7} 
                                         & KNN-1                    & 14.55  & 9.53  & 7.71  & 5.76  & 4.60  \\ \cline{2-7} 
                                         & KNN-3                    & 15.84  & 9.66  & 7.50  & 5.52  & 4.34  \\ \cline{2-7} 
                                         & KNN-5                    & 16.68  & 9.83  & 7.47  & 5.47  & 4.32  \\ \hline
\end{tabular}
\caption{Classification error (\%) for test data vs training data size  with 0\% added label noise for MNIST dataset.}
\label{cetest-datasize-labelnoise0-mnist}
\end{table}

\begin{table}[]
\centering
\begin{tabular}{|c|c|c|c|c|c|c|}
\hline
\multirow{2}{*}{added label noise = 10\%} & \multirow{2}{*}{Methods} & \multicolumn{5}{c|}{Data Size (MNIST)} \\ \cline{3-7} 
                                          &                          & 500    & 1500  & 3000  & 7500  & 15000 \\ \hline
\multirow{5}{*}{ce \% (test)}             & Gauss, Interpolation     & 20.98  & 16.32 & 14.00 & 13.01 & 12.16 \\ \cline{2-7} 
                                          & Laplace, Interploation   & 21.08  & 15.49 & 13.85 & 12.38 & 11.41 \\ \cline{2-7} 
                                          & KNN-1                    & 28.90  & 24.72 & 22.78 & 21.90 & 21.21 \\ \cline{2-7} 
                                          & KNN-3                    & 27.13  & 20.41 & 17.91 & 15.98 & 14.76 \\ \cline{2-7} 
                                          & KNN-5                    & 25.37  & 18.66 & 16.05 & 14.39 & 13.06 \\ \hline
\end{tabular}
\caption{Classification error (\%) for test data vs training data size  with 10\% added label noise for MNIST dataset.}
\label{cetest-datasize-labelnoise10-mnist}
\end{table}







\end{appendices}

\end{document}